\documentclass[11pt]{article}
\usepackage[T1]{fontenc}
\usepackage{lmodern}
\DeclareTextFontCommand{\textsc}{\normalfont\scshape}
\newcommand{\arxiv}[1]{#1}
\newcommand{\iclr}[1]{}
\PassOptionsToPackage{table,dvipsnames}{xcolor}

\usepackage{amsmath,amsfonts,bm}

\def\eqref#1{equation~\ref{#1}}

\def\1{\bm{1}}

\def\eps{{\epsilon}}

\DeclareMathAlphabet{\mathsfit}{\encodingdefault}{\sfdefault}{m}{sl}
\SetMathAlphabet{\mathsfit}{bold}{\encodingdefault}{\sfdefault}{bx}{n}

\arxiv{
\usepackage[a4paper, margin=2.5cm]{geometry}
\PassOptionsToPackage{hypertexnames=false}{hyperref}  
\usepackage[dvipsnames]{xcolor}
\usepackage[colorlinks]{hyperref}
\hypersetup{
    citecolor=[RGB]{50,100,170},
    linkcolor=[RGB]{50,100,170},
    urlcolor=[RGB]{255,102,178}}
\usepackage{fancyhdr}
}

\arxiv{
\newcommand{\toptitlebar}{%
  {\color{black}\hrule height 1pt}%
  \vskip 0.25in%
}
}

\makeatletter
\arxiv{
\renewcommand{\maketitle}{%
  \begingroup
  \renewcommand{\thefootnote}{}
  \begin{center}%
    \makebox[\textwidth][l]{\includegraphics[width=0.28\textwidth]{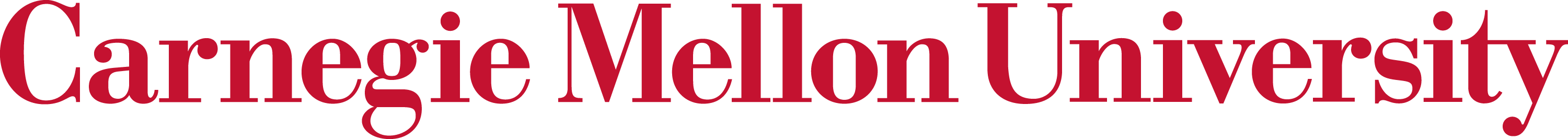}}\par
    \vskip 0.1in%
    \toptitlebar
    \vskip 0.1in%
    {\LARGE\bfseries \@title \par}%
    \vskip 0.3in%
    {\normalsize \@author \par}%
  \end{center}%
  \@thanks
  \endgroup
  \setcounter{footnote}{0}%
  \global\let\@thanks\@empty
  \par
  \vskip 0.3in%
}

\renewcommand\section{\@startsection {section}{1}{\z@}{-2.0ex plus
    -0.5ex minus -.2ex}{1.5ex plus 0.3ex minus .2ex}{\large\bfseries\raggedright}}
\renewcommand\subsection{\@startsection{subsection}{2}{\z@}{-1.8ex plus
    -0.5ex minus -.2ex}{0.8ex plus .2ex}{\normalsize\bfseries\raggedright}}
\renewcommand\subsubsection{\@startsection{subsubsection}{3}{\z@}{-1.5ex plus
   -0.5ex minus -.2ex}{0.5ex plus .2ex}{\normalsize\bfseries\raggedright}}

\renewenvironment{abstract}%
  {\centerline{\large\bfseries Abstract}%
   \begin{list}{}%
      {\setlength{\rightmargin}{0.6cm}%
       \setlength{\leftmargin}{0.6cm}}%
    \item[]\ignorespaces}%
  {\unskip\end{list}}

\setlength{\parindent}{2em}
\setlength{\parskip}{0.4em}
}
\makeatother

\iclr{
\usepackage[colorlinks=true, linkcolor=blue!70!black, citecolor=blue!70!black,urlcolor=blue!70!black,breaklinks=true]{hyperref}
\PassOptionsToPackage{dvipsnames}{xcolor} 
}

\usepackage{microtype}
\usepackage{hhline}

\makeatletter
\newcommand{\neutralize}[1]{\expandafter\let\csname c@#1\endcsname\count@}
\makeatother

\usepackage{algorithm}

\arxiv{
\usepackage{natbib}
\bibliographystyle{plainnat}
\bibpunct{(}{)}{;}{a}{,}{,}
}

\usepackage{amsthm}
\usepackage{mathtools}
\usepackage{amsmath}
\usepackage{bbm}
\usepackage{amsfonts}
\usepackage{amssymb}

\usepackage{xpatch}
\usepackage[nameinlink,capitalize]{cleveref}

\theoremstyle{plain}
\newtheorem{theorem}{Theorem}
\newtheorem{lemma}[theorem]{Lemma}
\newtheorem{assumption}[theorem]{Assumption}
\newtheorem{condition}[theorem]{Condition}

\AtBeginEnvironment{example}{\pushQED{\qed}}
\AtEndEnvironment{example}{\popQED}
\newtheorem{proposition}[theorem]{Proposition}

\theoremstyle{definition}

\AtBeginEnvironment{lemma}{\crefalias{theorem}{lemma}}
\AtBeginEnvironment{assumption}{\crefalias{theorem}{assumption}}
\AtBeginEnvironment{condition}{\crefalias{theorem}{condition}}
\AtBeginEnvironment{example}{\crefalias{theorem}{example}}
\AtBeginEnvironment{proposition}{\crefalias{theorem}{proposition}}
\AtBeginEnvironment{fact}{\crefalias{theorem}{fact}}
\AtBeginEnvironment{remark}{\crefalias{theorem}{remark}}

\theoremstyle{definition}  

\newtheorem{corollary}{Corollary}[section]

\theoremstyle{plain}
\newtheorem{definition}{Definition}[section]

\xpatchcmd{\proof}{\itshape}{\normalfont\proofnameformat}{}{}
\newcommand{\proofnameformat}{\bfseries}

\renewcommand{\eqref}[1]{\texorpdfstring{\hyperref[#1]{(\ref*{#1})}}{(\ref*{#1})}}

\crefformat{equation}{#2Eq. (#1)#3}
\Crefformat{equation}{#2Eq. (#1)#3}

\Crefformat{figure}{#2Figure #1#3}
\Crefname{assumption}{Assumption}{Assumptions}
\Crefformat{assumption}{#2Assumption #1#3}
\Crefname{subsubsection}{Section}{Sections}
\crefformat{subsubsection}{#2Section #1#3}
\Crefformat{subsubsection}{#2Section #1#3}
\Crefname{alg}{Alg.}{Algs.}

\usepackage{crossreftools}
\usepackage{xparse}

\ExplSyntaxOn
\DeclareDocumentCommand{\XDeclarePairedDelimiter}{mm}
 {
  \__egreg_delimiter_clear_keys: 
  \keys_set:nn { egreg/delimiters } { #2 }
  \use:x 
   {
    \exp_not:n {\NewDocumentCommand{#1}{sO{}m} }
     {
      \exp_not:n { \IfBooleanTF{##1} }
       {
        \exp_not:N \egreg_paired_delimiter_expand:nnnn
         { \exp_not:V \l_egreg_delimiter_left_tl }
         { \exp_not:V \l_egreg_delimiter_right_tl }
         { \exp_not:n { ##3 } }
         { \exp_not:V \l_egreg_delimiter_subscript_tl }
       }
       {
        \exp_not:N \egreg_paired_delimiter_fixed:nnnnn 
         { \exp_not:n { ##2 } }
         { \exp_not:V \l_egreg_delimiter_left_tl }
         { \exp_not:V \l_egreg_delimiter_right_tl }
         { \exp_not:n { ##3 } }
         { \exp_not:V \l_egreg_delimiter_subscript_tl }
       }
     }
   }
 }

\keys_define:nn { egreg/delimiters }
 {
  left      .tl_set:N = \l_egreg_delimiter_left_tl,
  right     .tl_set:N = \l_egreg_delimiter_right_tl,
  subscript .tl_set:N = \l_egreg_delimiter_subscript_tl,
 }

\cs_new_protected:Npn \__egreg_delimiter_clear_keys:
 {
  \keys_set:nn { egreg/delimiters } { left=.,right=.,subscript={} }
 }

\cs_new_protected:Npn \egreg_paired_delimiter_expand:nnnn #1 #2 #3 #4
 {
  \mathopen{}
  \mathclose\c_group_begin_token
   \left#1
   #3
   \group_insert_after:N \c_group_end_token
   \right#2
   \tl_if_empty:nF {#4} { \c_math_subscript_token {#4} }
 }
\cs_new_protected:Npn \egreg_paired_delimiter_fixed:nnnnn #1 #2 #3 #4 #5
 {
  \mathopen{#1#2}#4\mathclose{#1#3}
  \tl_if_empty:nF {#5} { \c_math_subscript_token {#5} }
 }
\ExplSyntaxOff

\XDeclarePairedDelimiter{\supnorm}{
  left=\lVert,
  right=\rVert,
  subscript=\infty
  }

\usepackage{url}
\usepackage{algorithmicx}
\usepackage{algpseudocode}
\usepackage{float}
\usepackage{graphicx}
\usepackage{wrapfig}
\usepackage{needspace}
\usepackage{enumitem}
\usepackage[most]{tcolorbox}
\usepackage{tabularx}
\usepackage{booktabs}
\usepackage{array}

\newtcolorbox{abstbox}[1]{%
  enhanced, breakable, colback=black!4, colframe=black, coltitle=white,
  colbacktitle=black, fonttitle=\bfseries\small, title={#1},
  fontupper=\footnotesize, arc=1.2mm, boxrule=0.9pt, left=3mm, right=3mm, top=2mm, bottom=2mm,
  before skip=8pt, after skip=8pt}
\newtcblisting{minesweeperpromptbox}[1]{%
  enhanced, breakable, colback=black!4, colframe=black, coltitle=white,
  colbacktitle=black, fonttitle=\bfseries\small, title={#1},
  arc=1.2mm, boxrule=0.9pt, left=3mm, right=3mm, top=2mm, bottom=2mm,
  before skip=8pt, after skip=8pt, listing only, listing engine=listings,
  listing options={basicstyle=\ttfamily\fontsize{7.5}{8.5}\selectfont,
    breaklines=true,
    breakatwhitespace=true, columns=fullflexible, keepspaces=true,
    showstringspaces=false}}
\newcommand{\abstelide}[1]{\par\smallskip\centerline{\textit{\footnotesize [\,#1 \,]}}\par\smallskip}

\definecolor{propblue}{RGB}{70,105,145}
\definecolor{propbluebg}{RGB}{245,248,252}
\definecolor{takeawayframe}{RGB}{164,127,62}
\definecolor{takeawaytitle}{RGB}{245,233,207}
\definecolor{takeawaybg}{RGB}{253,250,243}

\newtcbtheorem{propositionbox}{Proposition}%
{
    colback=takeawaybg,
    colframe=takeawayframe,
    colbacktitle=takeawaytitle,
    coltitle=black!85,
    fonttitle=\bfseries,
    boxrule=0.6pt,
    arc=2pt,
    left=6pt,
    right=6pt,
    top=6pt,
    bottom=6pt
}{prop}

\newif\ifdraftnotes
\draftnotestrue

\newtheorem*{proposition*}{Proposition}
\theoremstyle{definition}
\theoremstyle{plain}

\title{Training Language Models to Fail Gracefully\thanks{Project website, code, and other assets: \url{https://zanette-labs.github.io/failing-gracefully/}}}

\author{%
  \mbox{\textbf{Daman Arora}\textsuperscript{1}\thanks{Correspondence:
  \texttt{\{damana,azanette\}@andrew.cmu.edu}}}\quad
  \mbox{\textbf{Xuecheng Liu}\textsuperscript{1}}\quad
  \mbox{\textbf{Himanshu Gaurav Singh}\textsuperscript{2}}\\
  \mbox{\textbf{Fahim Tajwar}\textsuperscript{1}}\quad
  \mbox{\textbf{Abitha Thankaraj}\textsuperscript{1}}\quad
  \mbox{\textbf{Ruslan Salakhutdinov}\textsuperscript{1}}\quad\\
  \mbox{\textbf{Haiwen Feng}\textsuperscript{2}}\quad
  \mbox{\textbf{Yuda Song}\textsuperscript{1}}\quad
  \mbox{\textbf{Andrea Zanette}\textsuperscript{1}}\\[0.5em]
  \textsuperscript{1}Carnegie Mellon University\quad
  \textsuperscript{2}UC Berkeley%
}
\hypersetup{%
  pdftitle={Training Language Models to Fail Gracefully},
  pdfauthor={Daman Arora, Xuecheng Liu, Himanshu Gaurav Singh, Fahim Tajwar, Abitha Thankaraj, Ruslan Salakhutdinov, Haiwen Feng, Yuda Song, Andrea Zanette}%
}
\date{}
\begin{document}

\maketitle

\begin{abstract}

Reinforcement learning makes language models better at solving tasks, but the same mechanism also rewards guessing. When success is the only way to earn reward, opting for a safe fallback is always worse than a long-shot attempt. When a wrong action can corrupt a database or ship broken code, that incentive can be costly. Therefore, we propose training models to \textit{fail gracefully}: to preserve successful behavior whenever possible while learning to execute a safe fallback policy when acting would be unreliable. However, naively using policy gradients with intermediate reward for fallback behavior is brittle due to unfavorable optimization dynamics. We address this challenge with \textsc{GracefulRL}, which reshapes the optimization landscape to promote graceful failure. Importantly, this behavior emerges directly from outcome-level reward signals, without annotations specifying when to fall back or additional inference-time machinery such as a verifier. Across a range of tasks and scales, we train models that largely preserve their task-solving capability while selectively falling back when acting would be unreliable. More broadly, our results show that reinforcement learning can not only incentivize models to succeed, but also to reason about the limits of their capabilities and change course when they determine that success is out of reach.

\end{abstract}

\section{Introduction}
Large Language Models (LLMs) can now complete in minutes tasks that would take a person hours. They write and refactor code across whole repositories, produce long mathematical arguments, and act as agents that edit files and run commands on our behalf.  Significant improvements in such capabilities can be attributed to Reinforcement learning (RL) in domains with verifiable rewards such as coding and mathematics, ~\citep{Guo_2025,o1card,shao2024deepseekmathpushinglimitsmathematical} where models are rewarded for successful task completion \citep{lambert2025tulu,deepcoder2025}. 

These improvements are often measured by metrics such as accuracy or success rate, but these metrics say little about how the model \emph{behaves when it does not succeed}. Models are imperfectly informed and computationally bounded during inference. Therefore, there always exist situations in which appropriate actions cannot be reliably determined. In such cases, reinforcement learning can incentivize guessing, as a strategy to maximize reward \citep{kalai2025languagemodelshallucinate}, laying the burden of verification on the user. Someone needs to check the proof, review the pull request, and confirm what the agent actually did. This verification can be tedious and prone to errors, and when it fails, mistakes can propagate into consequential actions.

Therefore, we want models that change course when success is unlikely. If they can solve a task reliably, they should. If they cannot, they should do something else: revert their changes, inform the user that they are stuck, or hand the task to a different model. We call this behaviour \textit{failing gracefully}. Not only does graceful failure avoid harmful actions, it also helps build trust with users.

The naive way to train for this behavior is to reward success most, fallback second, and failure least. In principle, this should teach the model to keep trying when success is within reach and to fallback when it is not. Surprisingly, we find that learning this behavior with standard policy-gradient methods \citep{williams1992mlj-simple} and their derivatives \citep{shao2024deepseekmathpushinglimitsmathematical} can be brittle. In the worst case, they can collapse to a model that always falls back and stops solving the task altogether (Section \ref{sec:collapse}). We trace this failure to \textit{unfavorable optimization dynamics}: the objective expresses the behavior we want, but the optimization process does not reliably produce it.

We introduce \textsc{GracefulRL}, a reinforcement learning approach designed to avoid this pathology. We modify the optimization landscape which lets us maintain pressure to solve the task, while simultaneously encouraging the model to fall back rather than making an error when success is unlikely. Crucially, the model learns when to fall back from its own interaction with the environment with no labels specifying which problems are too hard, and no separate uncertainty model or post-hoc decision rule at inference time.

Across navigation, mathematical reasoning, and multi-turn Minesweeper, \textsc{GracefulRL} learns selective fallback while preserving strong task performance. On AIME, it reduces the error rate from 79.1\% to 8.9\% while sacrificing only 3.1 percentage points of success relative to standard RL. In Minesweeper, it improves both sides of the tradeoff, increasing success while reducing errors. Perhaps most strikingly, we find that the model verbalizes its decision making strategy in its chain-of-thought when its about to change course. Overall, our results suggest that reinforcement learning can not only improve a model's task performance, but also shape it's behavior when success is unlikely.

\begin{figure*}[t]
    \centering
    \includegraphics[width=\textwidth]{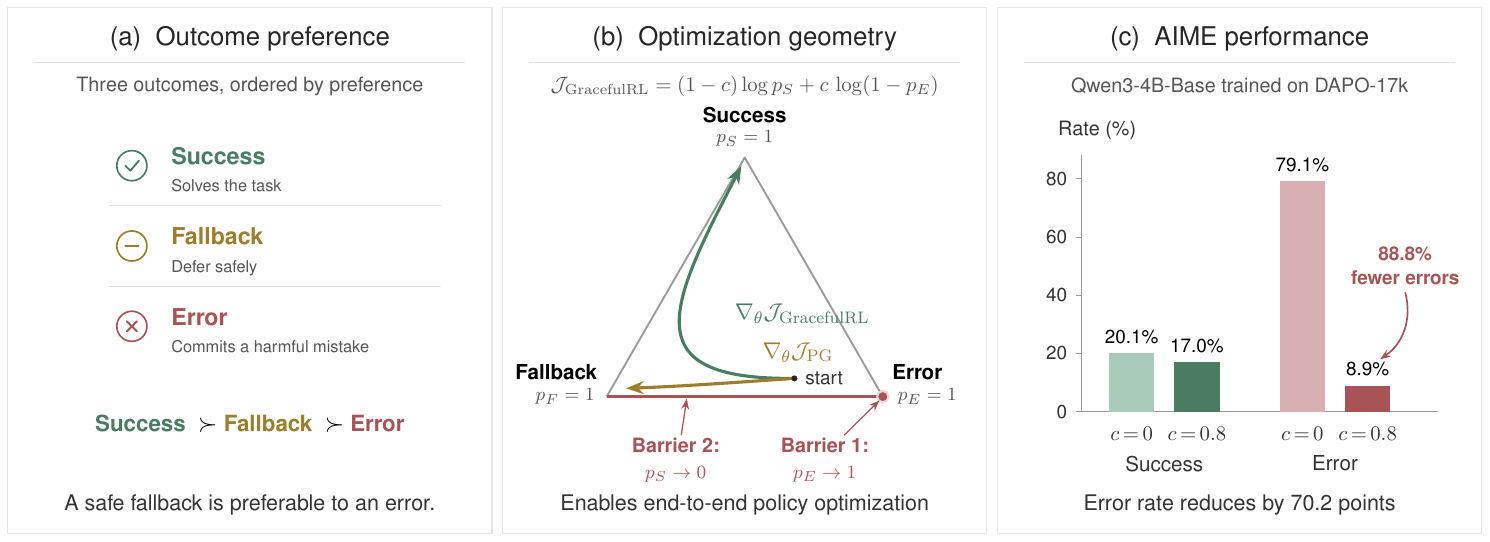}
    \caption{\textbf{Learning to fail gracefully.}
    \textbf{(a)} \textsc{GracefulRL} augments standard RLVR with an explicit fallback behavior, ordered as $\textsc{Success}\succ\textsc{Fallback}\succ\textsc{Error}$.
    \textbf{(b)} Standard policy gradient (PG) can lose optimization signal near saturated policies. In contrast, the \textsc{GracefulRL} objective retains a non-vanishing gradient. \textbf{(c)} Using models trained with \textsc{GracefulRL} on the AIME benchmark, the most cautious model ($c=0.8$) reduces the error rate from $79.1\%$ to $8.9\%$, an $88.8\%$ relative reduction, while success only drops by 3.1\% compared to RL with binary rewards ($c=0$).}
    \label{fig:main-teaser}
\end{figure*}

\section{Problem Definition and Setup}
\label{sec:setup}

\textbf{Notation.} Let $x\sim\mathcal D$ denote a task. A policy $\pi_\theta$, parameterized by
$\theta$, generates a sequence of actions
$y=(a_1,\ldots,a_T)\sim\pi_\theta(\cdot\mid x)$, where $T \leq H$ for some horizon $H$. The sequence may consist of language-model tokens or actions in an interactive environment. In standard RLVR, a verifier/reward function $f$ classifies each generated sequence as either a success or an error, i.e.,  $f(x,y)\in\{\textsc{Success},\textsc{Error}\}$. For a fixed task $x$, let
$p_S(x)=\mathbb{P}_{y\sim\pi_\theta(\cdot\mid x)}[f(x,y)=\textsc{Success}]$.
Standard RLVR seeks policy parameters that maximize the average success probability over the task distribution, i.e.,
$\theta^\star\in\arg\max_\theta
\mathbb E_{x\sim\mathcal D}[p_S(x)]$.
For this, many different algorithms have been proposed \citep{schulman2017ppo,shao2024deepseekmathpushinglimitsmathematical,ahmadian-etal-2024-back,tajwar2026maximum}.

\textbf{Success, Fallback, and Error.} In \textsc{GracefulRL}, we augment the outcome space for model trajectories with an explicit fallback option, such that $f(x,y)\in\{\textsc{Success},\textsc{Error},\textsc{Fallback}\}$. The fallback outcome represents a more favorable option compared to failing. It's semantics would be environment specific: it may correspond to abstaining on a reasoning problem, restoring a database to a previous state, or handing off a task to a different model/sub-agent. We encode a preference ordering
$\textsc{Success}\succ\textsc{Fallback}\succ\textsc{Error}$ using the following ternary reward, where $c$ controls how strongly errors are avoided relative to successes:

\begin{equation*}
    r(x,y)=
    \begin{cases}
        1, & f(x,y)=\textsc{Success},\\
        c, & f(x,y)=\textsc{Fallback},\\
        0, & f(x,y)=\textsc{Error},
    \end{cases}
    \qquad 0<c<1.
\end{equation*} 
\textbf{Evaluating Graceful Failure.} In order to evaluate the performance of a model we define success, fallback and error rates. Concretely, given a test distribution, $\mathcal D_{\mathrm{test}}$,
we define:
\begin{align*}
    \operatorname{Success@}1
    &=\mathbb{P}_{x\sim\mathcal D_{\mathrm{test}},\,
    y\sim\pi_\theta(\cdot\mid x)}[f(x,y)=\textsc{Success}],\\
    \operatorname{Error@}1
    &=\mathbb{P}_{x\sim\mathcal D_{\mathrm{test}},\,
    y\sim\pi_\theta(\cdot\mid x)}[f(x,y)=\textsc{Error}].
\end{align*} 
where $\operatorname{Fallback@}1 = 1-\operatorname{Success@}1-\operatorname{Error@}1$. A method $\mathcal{A}$ \textit{Pareto-dominates} a method
$\mathcal{B}$ if it achieves a higher or equal $\operatorname{Success@}1$ for a lower or equal $\operatorname{Error@}1$ with a strict improvement in at least one metric.

\section{The Challenge of Optimizing Fallback Behavior}
\label{sec:method}
\label{sec:collapse}

A natural approach is to use the ternary reward in Section \ref{sec:setup} in a policy-gradient algorithm which maximizes the expected reward, that is maximize $
\mathcal{J}_{\mathrm{PG}}(x)
=p_S(x)+c\,p_F(x).$
Unfortunately, we find that this approach is highly sensitive to initialization and reward magnitude, necessitating tuning for every model and task combination. In the worst case, the policy can converge to a degenerate one that always executes the fallback with no task progress, or doesn't learn to fallback at all. 

\textbf{Vanishing gradients near saturated policies.} Why does such a natural objective exhibit this behavior? Formally, we show that this pathology emerges because the policy gradient vanishes when the policy approaches saturation. This implies that if the policy saturated on suboptimal outcomes during training, it can lead to poor convergence. We make a formal claim.

\begin{propositionbox}
{Gradient suppression near saturated policies [\textit{Proof: App.~\ref{app:collapse}}]}
{gradient-suppression}
Under the regularity conditions mentioned in Appendix \ref{app:collapse}, for all possible outcomes of the policy
$o\in\{\textsc{Success},\textsc{Fallback},\textsc{Error}\}$, \[
p_o\rightarrow 1\quad \Rightarrow \quad  \left\|\nabla_\theta \mathcal{J}_{\mathrm{PG}}(x)\right\|
=
O\!\left(1-p_o(x)\right).
\]
\end{propositionbox}

Intuitively, as the policy places nearly all of its probability on a single outcome, the gradient shrinks proportionally to the probability mass remaining on the other outcomes. Thus, once the policy becomes nearly deterministic, for example, always falling back, the gradient update required to move probability back toward success becomes vanishingly small. This result applies to standard softmax-parameterized policies, including auto-regressive language models, under a mild boundedness condition on the gradients of the logits. For a detailed exposition, refer to Appendix \ref{app:collapse}.

\section{Optimizing for Graceful Failure}
\label{sec:main_method}
This suggests that directly optimizing the objective of interest may not be the ideal decision, motivating the use of a surrogate objective with more favorable optimization properties. The distinction between the objective of interest and the objective used for optimization is ubiquitous in ML. In classification, for example, the quantity of interest is often the zero-one error, while training generally uses a cross entropy loss that generates more useful gradients \citep{tong,Bartlett01032006}. In particular, assigning fallback an intermediate scalar reward and using the expected reward as the objective \textit{encodes the right behavior but says little about whether it is easy to optimize}.

Section \ref{sec:collapse} shows that the expected-reward objective can provide a poor optimization landscape. A natural way to prevent either desirable event from becoming vanishingly unlikely is to maximize
its log-probability. Now, consider the 2 events: 
\[ 
E_1: \{\textsc{Success}\} \quad \quad \quad E_2: \neg\{\textsc{Error}\}
\]
Note that $\neg\textsc{Error} = \{\textsc{Success}, \textsc{Fallback}\}$. We jointly maximize the log-probability of both events using a weighted sum:
\[
\mathcal{J}_{\textsc{GracefulRL}}= 
(1-c)\cdot \log \mathbb{P}(E_1) + c \cdot \log \mathbb{P}(E_2) = (1-c)\log p_S(x)
+
c\log\left(1-p_E(x)\right).
\]

At $c=0$, the objective reduces to $\log p_S(x)$ where only success matters and fallback and error are treated equivalently. At $c=1$, it reduces to $\log \left(1-p_E(x)\right)$: which avoids errors, and success and fallback are treated equivalently. When $0<c<1$, we get an interpolation between these extremes.

\textbf{A log-barrier perspective.} This can be viewed as placing logarithmic barriers on two undesirable boundaries of the outcome simplex, $p_E \rightarrow 1$ and $p_S \rightarrow 0$ (see Figure \ref{fig:main-teaser}\textbf{(b)}). Both terms in the objective diverge when $p_E\rightarrow1$, whereas only the first diverges when $p_E+p_F\rightarrow1$. The objective is maximized when $p_S\rightarrow 1$. In simpler words, this induces a natural ordering, where success is preferred over fallback which is in turn preferred over failure. Now, we make a formal claim.

\begin{figure}[t]
    \centering
    \includegraphics[width=\textwidth]{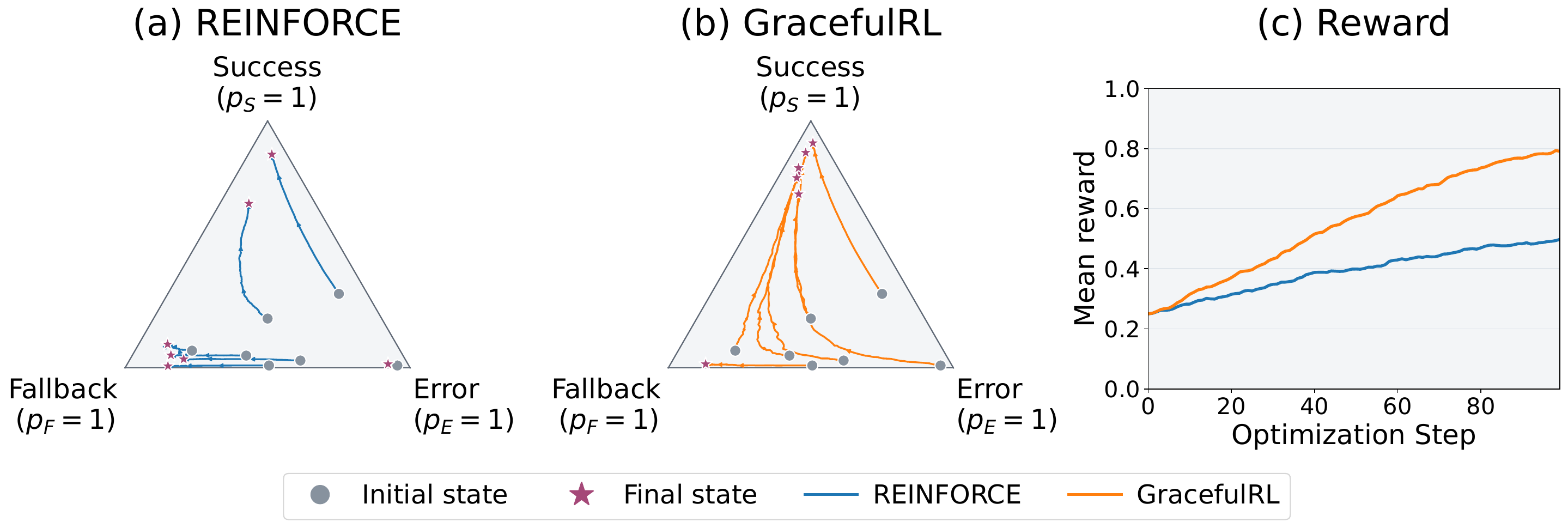}
    \caption{\textbf{A simple 3-class softmax policy.} We use logits $z_S, z_F$ and $z_E$ to create a policy $\pi = (p_S, p_F, p_E) = \mathrm{softmax}(z_S, z_F, z_E)$. \textbf{(a)} For initializations with low initial success probabilities, $\nabla_\theta\mathcal{J}_{pg}$ leads to traps near the fallback and error vertices ($p_F=1, p_E=1$),  \textbf{(b)} however, $\nabla_\theta\mathcal{J}_{\textsc{GracefulRL}}$ is more robust to initialization leading to higher success and \textbf{(c)} higher reward.}
    \label{fig:sgd-trajectories}
\end{figure}

\begin{propositionbox}
{Non-vanishing gradients near saturated policies [\textit{Proof: App.~\ref{app:collapse}}]}
{gradient-nonsuppression}
Under the regularity and non-degeneracy assumptions in Appendix \ref{app:collapse}, for either \textbf{sub-optimal} outcome $o\in\{\textsc{Fallback},\textsc{Error}\}$,
\[
p_o\rightarrow1\quad \Rightarrow\quad \left\|\nabla_\theta \mathcal{J}_{\mathrm{\textsc{GracefulRL}}}(x)\right\|
=
\Theta(1).
\]

\end{propositionbox}

In words, even when the policy's fallback or error probability gets high the gradient remains bounded away from zero, preserving a signal that can move the policy away from these degenerate solutions. Figure \ref{fig:sgd-trajectories} illustrates the differences between the methods. Standard policy gradient can get trapped near the fallback or error vertex when initialized with low success probability. Whereas, \textsc{GracefulRL} retains a recovery signal and moves toward higher success. When success is initially unlikely, it shifts probability from error to fallback before eventually moving toward success.

\subsection{The \textsc{GracefulRL} algorithm} 

\begin{wrapfigure}{r}{0.54\textwidth}
\refstepcounter{algorithm}
\label{alg:graceful-rl}
\vspace{-1.5em}
\textbf{Algorithm~\thealgorithm:}\quad \textsc{GracefulRL}
\vspace{0.3em}
\hrule

\begin{algorithmic}[1]
\Require Tasks $B$; rollouts $N$; policy $\pi_\theta$; verifier $f$;
fallback utility $c\in(0,1)$; smoothing constant $\epsilon$
\For{each task $x\in B$}
    \State Sample $y_{1:N}\sim\pi_\theta(\cdot\mid x)$
    \State $f_n\gets f(x,y_n)$ for $n=1,\ldots,N$
    \State $N_S\gets\sum_n\mathbf{1}\{f_n=\textsc{Success}\}$
    \State $N_F\gets\sum_n\mathbf{1}\{f_n=\textsc{Fallback}\}$
    \For{$n=1,\ldots,N$}
    \State $\displaystyle
            \begin{aligned}
            a_n \gets{}&
            \dfrac{(1-c)\mathbf{1}\{f_n=\textsc{Success}\}}
                  {N_S+\epsilon} \\
            &+
            \dfrac{c\mathbf{1}\{f_n\ne\textsc{Error}\}}
                  {N_S+N_F+\epsilon}
            \end{aligned}$
    \EndFor
    \State $\bar a\gets N^{-1}\sum_n a_n$
    \State $\widehat g(x)\gets N^{-1}\sum_n(a_n-\bar a)
    \nabla_\theta\log\pi_\theta(y_n\mid x)$
\EndFor
\State \Return $|B|^{-1}\sum_{x\in B}\widehat g(x)$
\end{algorithmic}
\hrule
\end{wrapfigure}

Conveniently, the objective derived above is the TailRL objective \citep{ramasubramanian2026taillikelihoodreinforcementlearning} specialized to the ternary reward $r\in\{0, c, 1\}$. While TailRL motivates the objective from a tail-likelihood and improving $\operatorname{Best@}K$ perspective, we provide a re-interpretation through the log barrier perspective in Section \ref{sec:main_method}, which highlights how the objective reshapes the loss landscape and prevents vanishing gradients. We make one change to the TailRL algorithm by adding a smoothing constant $\eps$. A small $\eps>0$ shrinks the large inverse-frequency weights that arise when success or non-error trajectories are rare which could reduce variance at the cost of some bias. We discuss the effects of smoothing in Appendix \cref{app:eps_theory,app:safemaze-epsilon,app:llm-epsilon}. Algorithm \ref{alg:graceful-rl} provides a reference implementation.

\section{Experiments}
Motivated by our theoretical insights, we study \textsc{GracefulRL} in three practical settings. Our first setting involves a maze navigation task with a small ResNet \citep{He_2016_CVPR} policy. This gives us a compute-efficient setup that allows us to scale training steps and examine the effects of different hyperparameters. Next, we move to a relevant problem: training LLMs to abstain on mathematical reasoning problems. Here we test whether \textsc{GracefulRL} scales to more realistic scenarios and also test meaningful generalization across different domains. Finally, we design an experiment where an LLM has to make irreversible sequential decisions in the Minesweeper game, where we discover that training for graceful failure can sometimes improve task performance as well.

\subsection{SafeMaze: A Controlled Testbed}
\label{sec:safemaze}
\begin{figure}[t]
    \centering
    \includegraphics[width=\textwidth]{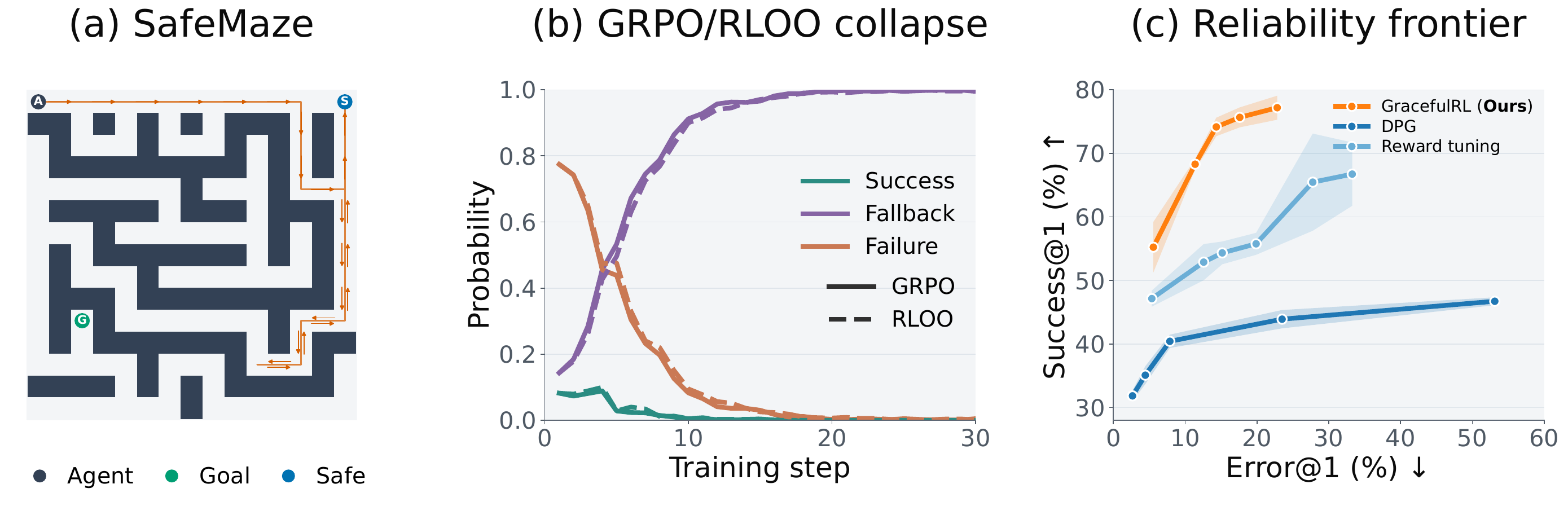}
    \caption{\textbf{(a)} The \textsc{SafeMaze} task requires the agent to reach the \textsc{Goal}, but if it can't, it should backtrack to \textsc{Safe}. \textbf{(b)} GRPO and RLOO can collapse to a degenerate policy $p_F = 1$ that is, always going to the \textsc{Safe} square. \textbf{(c)} \textsc{GracefulRL} is robust and Pareto-dominates reward-tuning and
DPG (across 3 seeds), achieving higher a $\operatorname{Success@}1$ rate at a lower $\operatorname{Error@}1$ rate.} 
    \label{fig:safemaze-summary}
\end{figure}

\textbf{Environment.}
For our first experimental setting, we introduce \textsc{SafeMaze}, a maze navigation task where an agent can either reach the \textsc{Goal}, which counts as \textsc{Success}, or navigate to a designated \textsc{Safe} square, which counts as \textsc{Fallback}. Crucially, the \textsc{Safe} square is deliberately much easier to reach: it is connected to the start by an always-clear passage, whereas reaching the \textsc{Goal} requires solving the maze within $50$ steps or less. The mazes are designed in a way that \textsc{Goal} is always reachable (see Figure~\ref{fig:safemaze-summary} for an example). Hitting a wall or timing out counts as an \textsc{Error}. Thus, the agent faces exactly the tradeoff of interest: \textit{pursue the harder task for the possibility of success, or take an easy and reliable fallback that avoids error but forgoes success}. We use a 4M parameter, slightly pre-trained residual CNN as the policy here. More details are given in Appendix \ref{app:safemaze}.

\subsubsection{Results on \textsc{SafeMaze}}

\textbf{Standard policy gradient methods are brittle.} First, we train GRPO and RLOO on the \textsc{SafeMaze} task with a fixed value of the reward for reaching \textsc{Safe}, set to 0.4. We observe that, neither of these variants learn a meaningful policy. Within a few gradient steps, the policy reaches a $100\%$ fallback rate and stops solving the task (Figure \ref{fig:safemaze-summary}). In general, we observe that this behavior is \textit{highly sensitive to the fallback reward} $c$. In Appendix \ref{app:safemaze-fixed-rloo}, using $c=0.1$ for RLOO drives the fallback probability toward zero. At $c=0.2$, training retains some success but quickly plateaus. For every tested $c \geq 0.3$, the policy rapidly collapses to the fallback-only solution. Thus, useful behavior only occurs in a narrow range of fallback rewards rather than emerging robustly from the desired reward ordering. Appendix~\ref{app:safemaze-entropy} shows that entropy bonuses also do not help in improving performance.

\textbf{Comparing with other baselines.} Next, we compare with two stronger baselines. We observe that a  meaningful policy can be learnt with RLOO when the reward for \textsc{Safe} is lower than the success rate of the policy. Motivated by this at time step $t$, online reward tuning uses $c=\alpha\cdot \widehat p_{S,t}$ where $\widehat p_{S,t}$ denotes the online success rate. We also test Delightful Policy Gradient (DPG)~\citep{osband2026delightful} that modifies policy-gradient updates to mitigate trapping near deterministic sub-optimal policies \citep{mei2026delightfulgradientsacceleratecorner}. Figure~\ref{fig:safemaze-summary} shows the reliability curve for each method after $20$k training steps for different hyperparameters. We see a clear trend: \textbf{\textsc{GracefulRL} converges to a higher} $\operatorname{Success@}1$ at lower $\operatorname{Error@}1$ compared to the baselines. \textit{Thus, \textsc{GracefulRL} is robust to reward magnitudes unlike standard policy gradient and learns useful fallback behavior when uncertain.} 

Finally, we also investigate the interplay of various hyperparameters, specifically the role of the number of rollouts $N$, the fallback reward $c$, and the smoothing parameter $\eps$ in Appendix~\ref{app:safemaze-rollouts} and Appendix~\ref{app:safemaze-epsilon}. We observe that performance improves with more rollouts $N$, particularly at higher values of $c$, and also that a higher $\eps$ value gives higher reward at smaller compute budgets.

\subsection{Language-Model RL for Mathematical Reasoning}
\label{sec:math-reasoning}

\begin{figure}[t]
    \centering
    \includegraphics[width=\textwidth]{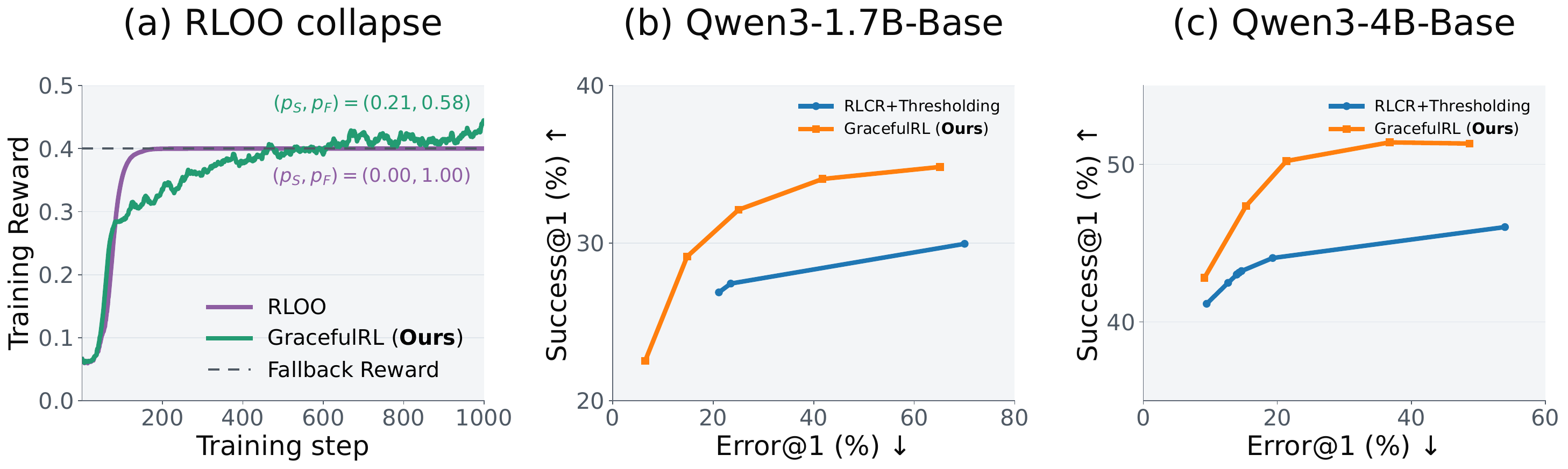}
    \caption{\textbf{(a)} Training for abstention using standard policy gradient (RLOO) can converge to a policy that always abstains, while our method learns a meaningful policy. \textbf{(b), (c)} Reliability frontiers for the trained Qwen3-1.7B-Base and Qwen3-4B-Base models, comparing RLCR+Thresholding and \textsc{GracefulRL}. \textsc{GracefulRL} achieves a higher success rate at a lower error rate. }
    \label{fig:rlcr-aggregate}
\end{figure}

Another important instance of graceful failure is \emph{abstention}. We study this setting in mathematical reasoning, our first LLM testbed. Here, a model must make two decisions: how to solve a problem, and whether it should attempt a final answer at all. For problems within its capabilities, the model should reason through them and produce a solution. For problems beyond its current competence, however, continuing to reason toward a final answer can lead to a confident but incorrect prediction. A graceful model should \textit{instead recognize when its reasoning is unlikely to succeed and abstain}. We therefore ask whether outcome-level RL can teach a language model not only to solve problems, but also to recognize when to stop trying. We first describe the experimental setup and the baseline.

\textbf{RLCR+Thresholding baseline.}
We compare against Reinforcement Learning with Calibration Rewards (RLCR)~\citep{damani2025beyond}, a recent uncertainty quantification approach, which trains language models for calibration. In RLCR, a model is trained to generate an answer and a confidence score \(q\in[0,1]\), and is then trained using a Brier-score based reward, that is, for a response $y$ for prompt $x$, 
$r_{\mathrm{RLCR}} = 
\mathbb{I}_\textrm{correctness}(y, x) - 
\left(\mathbb{I}_\textrm{correctness}(y, x)-q\right)^2$  \footnote{We give a reward of $-1.1$ for malformed outputs, since $-1$ is the minimum value of $r_\textrm{RLCR}.$}. To obtain a policy that abstains, we use a test time threshold $\tau$, where the model answers if confidence $q\geq\tau$ and abstains otherwise.

\textbf{Conditioned training (ours).} For \textsc{GracefulRL}, we jointly train a single model across five fallback utilities $c\in\{0,0.2,0.4,0.6,0.8\}$. We append a \textit{conditioning suffix} to each problem indicating the desired degree of caution. For $c=0$, it reads ``Be very optimistic and always attempt an answer, even if you are unsure,'' while for $c=0.8$ it reads ``Be quite careful and avoid guessing.'' The complete set appears in Appendix Table~\ref{tab:conditioning-suffixes}.
During training, we mix examples from all five conditioning levels and optimize each using the corresponding value of $c$, where correct answers, abstentions, and incorrect answers correspond to the \textsc{Success}, \textsc{Fallback}, and \textsc{Error} outcomes, respectively.

\textbf{Training and evaluation.}
We train the Qwen3-1.7B-Base and Qwen3-4B-Base \citep{yang2025qwen3technicalreport} models on the DAPO-Math-17K dataset \citep{yu2025dapo} using on-policy RL for $1{,}500$ steps and $1{,}000$ steps respectively. For evaluation, we use MATH500 \citep{hendrycks2021measuring}, AIME 2024-2026, AMC 2022-2024, and JEEBench (numeric) \citep{arora-etal-2023-llms}.  At test time, for RLCR, we sweep the threshold $\tau$ to trace its reliability curve, whereas for \textsc{GracefulRL}, we evaluate on each unique conditioning prompt, to measure the policy's reliability. Details are given in Appendix \ref{app:math-training}.

\subsubsection{Results on Language-Model RL}
\textbf{Pareto-dominance over the baseline.} From Figure \ref{fig:rlcr-aggregate} , our conditioned model consistently Pareto-dominates the baseline: it gets a higher success rate at a lower error rate across the entire frontier (Figure~\ref{fig:rlcr-aggregate}). The effect is especially strong on AIME for the 4B trained model (Figure \ref{fig:main-teaser}), where the error rate is reduced from $\sim79\%$ to $\sim9\%$ while sacrificing only about 3\% points of success compared to the $c=0$ conditioning which corresponds to standard RLVR.  We again note in Figure \ref{fig:rlcr-aggregate} that vanilla policy-gradient (RLOO) can collapse to a degenerate policy where for $c=0.4$, the Qwen3-1.7B-Base model trained using RLOO converges to a degenerate policy which always abstains on the training dataset, while \textsc{GracefulRL} continues to improve and achieves a success rate of $21\%$ with an abstention rate of $58\%$ at the end of training. This mirrors our finding in the \textsc{SafeMaze} setting. An online reward-tuning baseline can also lose abstention behavior during training (Appendix~\ref{app:llm-reward-tuning}). We show additional results in Appendix \cref{app:additional_lm_results,app:rlcr-per-dataset,app:llm-epsilon,app:rlcr-accuracy-calibration,app:llm-compute}.

\begin{figure}[t]
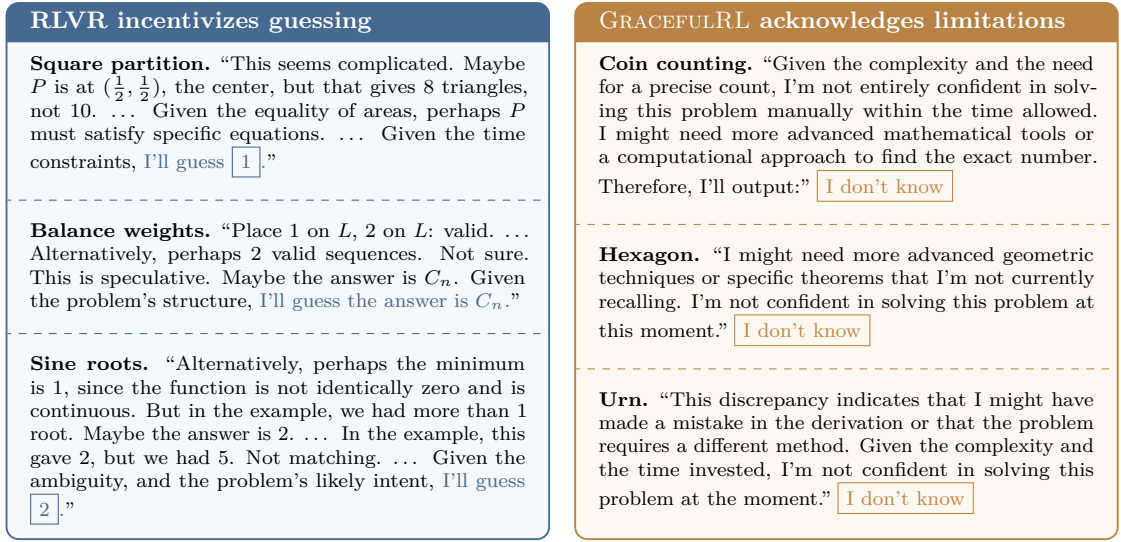

\centering
\begin{minipage}[t]{0.45\textwidth}
\vspace{0pt}
\begin{tcolorbox}[
  enhanced, width=\linewidth, equal height group=arxiv-cot,
  colback=propbluebg, colframe=propblue,
  colbacktitle=propblue, coltitle=white,
  title={RLVR incentivizes guessing}, fonttitle=\bfseries\footnotesize,
  fontupper=\fontsize{7.5}{8.8}\selectfont,
  fontlower=\fontsize{7.5}{8.8}\selectfont,
  boxrule=0.7pt, arc=1.5mm, left=2mm, right=2mm, top=0.8mm, bottom=0.8mm,
  before skip=0pt, after skip=0pt]
\textbf{Square partition.} ``This seems complicated. Maybe $P$ is at $(\frac12,\frac12)$, the center, but that gives 8 triangles, not 10. \ldots{} Given the equality of areas, perhaps $P$ must satisfy specific equations. \ldots{} Given the time constraints, \textcolor{propblue}{I'll guess $\boxed{1}$.}''
\tcblower
\textbf{Balance weights.} ``Place 1 on $L$, 2 on $L$: valid. \ldots{} Alternatively, perhaps 2 valid sequences. Not sure. This is speculative. Maybe the answer is $C_n$. Given the problem's structure, \textcolor{propblue}{I'll guess the answer is $C_n$.}''
\tcbline
\textbf{Sine roots.} ``Alternatively, perhaps the minimum is 1, since the function is not identically zero and is continuous. But in the example, we had more than 1 root. Maybe the answer is 2. \ldots{} In the example, this gave 2, but we had 5. Not matching. \ldots{} Given the ambiguity, and the problem's likely intent, \textcolor{propblue}{I'll guess $\boxed{2}$.}''
\end{tcolorbox}
\end{minipage}\hspace{0.02\textwidth}%
\begin{minipage}[t]{0.45\textwidth}
\vspace{0pt}
\begin{tcolorbox}[
  enhanced, width=\linewidth, equal height group=arxiv-cot,
  colback=BurntOrange!5, colframe=BurntOrange!70!black,
  colbacktitle=BurntOrange!70!black, coltitle=white,
  title={\textsc{GracefulRL} acknowledges limitations}, fonttitle=\bfseries\footnotesize,
  fontupper=\fontsize{7.5}{8.8}\selectfont,
  fontlower=\fontsize{7.5}{8.8}\selectfont,
  boxrule=0.7pt, arc=1.5mm, left=2mm, right=2mm, top=0.8mm, bottom=0.8mm,
  before skip=0pt, after skip=0pt]
\textbf{Coin counting.} ``Given the complexity and the need for a precise count, I'm not entirely confident in solving this problem manually within the time allowed. I might need more advanced mathematical tools or a computational approach to find the exact number. Therefore, I'll output:''\nobreak\hspace{0.3em}{\color{BurntOrange!75!black}$\boxed{\text{I don't know}}$}
\tcblower
\textbf{Hexagon.} ``I might need more advanced geometric techniques or specific theorems that I'm not currently recalling. I'm not confident in solving this problem at this moment.''\nobreak\hspace{0.3em}{\color{BurntOrange!75!black}$\boxed{\text{I don't know}}$}
\tcbline
\textbf{Urn.} ``This discrepancy indicates that I might have made a mistake in the derivation or that the problem requires a different method. Given the complexity and the time invested, I'm not confident in solving this problem at the moment.''\nobreak\hspace{0.3em}{\color{BurntOrange!75!black}$\boxed{\text{I don't know}}$}
\end{tcolorbox}
\end{minipage}
\caption{\textbf{Different behaviors under uncertainty.} A model trained with RLVR has no incentive to acknowledge limitations, and guesses when uncertain. But a model trained with \textsc{GracefulRL} verbalizes and acknowledges its lack in capabilities. Full traces appear in Appendix~\ref{app:abstention-reasons}.}
\label{fig:cot-contrast}
\end{figure}

\textbf{Explicit verbalization in the CoT.} Strikingly, CoT generated by the model trained with \textsc{GracefulRL} explicitly verbalizes notions of uncertainty and chooses to use the fallback when uncertain.
Figure~\ref{fig:cot-contrast} show three illustrative `semantic' patterns: hitting a time limit, being unable to recall a technique, and abstaining when it finds a mistake in its reasoning. In contrast, a standard RLVR-trained model\footnote{We train a separate Qwen3-4B-Base with a neutral prompt with no cues about guessing/being cautious, and train using MaxRL.} has no incentive to abstain and may continue with a speculative answer. Such guessing behavior is undesirable because it shifts the burden of verification onto the user. Moreover, signals that a model is guessing need not be verbalized, could be difficult to recognize, or expressed in a form/language that is hard to interpret. Therefore, explicitly training the model for fallback behavior provides a more robust strategy.

\Needspace{16\baselineskip}
\begin{wrapfigure}{r}{0.27\textwidth}
\centering
\includegraphics[width=\linewidth]{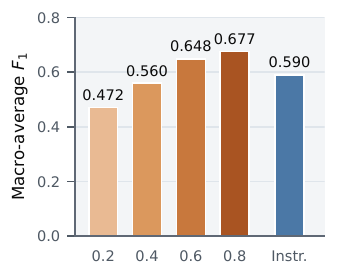}
\caption{\raggedright Abstention $F_1$: \textsc{GracefulRL} (orange), Instruct (blue).}
\label{fig:abstention-f1-aggregate}
\end{wrapfigure}

\textbf{Out-of-distribution transfer.} We evaluate our model trained on math problems on AbstentionBench~\citep{kirichenko2025abstentionbench}, which contains unknown, underspecified, and otherwise unanswerable questions from diverse domains such as History, MedicalQA, Geography and MoralQA. We find that the model retains its intended behavior. More cautious suffixes produce more fallbacks on questions that can't be answered reliably, while optimistic suffixes attempt more responses. The average abstention $F_1$ score rises from $0.472$ at $c=0.2$ to $0.677$ at $c=0.8$ outperforming the Qwen3-4B-Instruct-2507 model which scores $0.590$ (Figure~\ref{fig:abstention-f1-aggregate}). This indicates hints of generalization to domains other than mathematical reasoning. Detailed results on the entire benchmark appear in Appendix~\ref{app:abstentionbench}.

\subsection{Multi-Turn RL in Minesweeper}

Finally, we use an LLM as an agent in a multi-turn Minesweeper game~\citep{tajwar2025training}. This setting gives us a testbed that involves a \textit{sequential decision-making} environment with irreversible actions. The environment consists of a grid containing hidden mines and revealed cells which indicate the number of mines in their neighboring cells. At each step, the agent observes the current  board state and selects a cell to reveal. The game has $N$ turns and the rules are as follows:
\begin{enumerate}
    \item Revealing a mine immediately ends the game with an \textsc{Error}.
    \item Successfully revealing all non-mine cells by the end of the game, results in 
    \textsc{Success}.
    \item Avoiding mines but not revealing all non-mine squares by the end is a \textsc{Fallback}.
\end{enumerate}  
The \textsc{Fallback} behavior promotes conservative actions, even if it means giving up on completing the puzzle. Minesweeper pushes graceful failure beyond one-shot uncertainty estimation. The agent must repeatedly decide whether the available evidence justifies another action, with every move changing both the available information and the uncertainty of subsequent decisions. Consequently, methods such as RLCR or majority-vote thresholding, which attach confidence to a single generated answer, \textit{do not directly apply in this scenario}. This makes Minesweeper a natural setting for \textbf{end-to-end policy optimization from outcome-level rewards}. We examine whether \textsc{GracefulRL} can teach the agent \emph{when to keep playing and when continued action is no longer worth the risk}.

\textbf{Training and Evaluation.} We train Qwen3-4B-Instruct-2507 for 250 RL steps on puzzles requiring $N=8$ steps to solve and evaluate on $500$ held-out puzzles meant to be solved at a $N=12$-turn budget. Appendix~\ref{app:minesweeper} specifies the board construction, training configuration, and evaluation protocol.
\subsubsection{Results on Minesweeper}

\begin{figure}[t]
    \centering
    \begin{minipage}[c]{0.51\textwidth}
        \centering
        \small
        \begin{tabular}{@{}cccc@{}}
            \toprule
            Method & $c$ & $\operatorname{Success@}1(\uparrow)$ & $\operatorname{Error@}1(\downarrow)$ \\
            \midrule
            MaxRL & $0$ & $67.4$ & $17.0$ \\
            GracefulRL & $0.25$ & $66.0$ & $13.4$ \\
            GracefulRL & $0.5$ & $\mathbf{77.4}$ & $\mathbf{6.0}$ \\
            GracefulRL & $0.75$ & $73.6$ & $\mathbf{5.6}$ \\
            \rowcolor{gray!15}
            RLOO & $0.5$ & $0.0^{\dagger}$ & $0.0^{\dagger}$ \\
            \bottomrule
        \end{tabular}
        \vspace{0.3em}
        {\\$\dagger$ Collapses to always taking the fallback action.}
    \end{minipage}\hspace{0.025\textwidth}%
    \begin{minipage}[c]{0.35\textwidth}
        \centering
        \includegraphics[width=\linewidth]{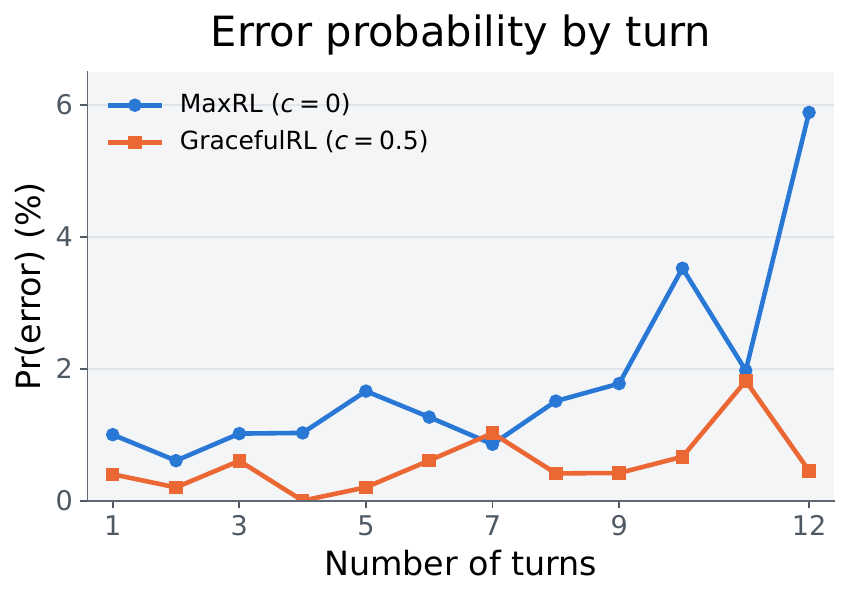}
    \end{minipage}
    \caption{\textbf{Minesweeper results.} \textbf{(a)} $\operatorname{Success@}1$ and $\operatorname{Error@}1$ rates after training the Qwen3-4B-Instruct-2507 model. In this setting, \textsc{GracefulRL} achieves a lower rate \textbf{and} higher performance compared to RL with no fallback option. \textbf{(b)} Probability of an error on each turn for MaxRL ($c=0$) and GracefulRL ($c=0.5$). Number of errors significantly increase at the end of the game for MaxRL since the incentive to `play it safe' reduces and risk taking moves offer better expected return.}
    \label{fig:minesweeper-results}
\end{figure}

\textbf{Graceful failure can (sometimes) improve task success.} In the preceding settings,
greater caution generally trades some success for fewer errors. However, in Minesweeper, varying $c$ reveals a surprising finding: a fallback reward can also improve task success (Figure~\ref{fig:minesweeper-results}a). At $c=0.5$, \textsc{GracefulRL} achieves $77.4\%$ success and $6.0\%$ error, \textbf{improving both success rate and lowering error rate}  over MaxRL. We hypothesize that in this setting, learning a reliable fallback steers the model toward safer board states from which successful completion is more accessible. In this sense, the fallback reward may act as a form of reward shaping. RLOO with $c=0.5$ collapses achieving neither successes nor errors mirroring our findings in the Sections \ref{sec:safemaze} and \ref{sec:math-reasoning}.

\textbf{Errors concentrate late in the game.} We also observe a striking pattern: for the MaxRL model ($c=0$), errors become significantly more frequent in the latter half of the 12-turn game (Figure~\ref{fig:minesweeper-results}b). This behavior is consistent with the underlying incentive structure. As the game progresses and successful completion becomes less likely, a safe action that merely avoids failure receives no incentive, \textit{whereas a risky move retains some chance of producing a successful trajectory}. The resulting incentive increasingly favors `gambling' on low-probability success over preserving a safe state. In contrast, \textsc{GracefulRL} ($c=0.5$) explicitly rewards safe fallback behavior, and the probability of error remains markedly lower near the end of the episode, including on the final turn.

\section{Related Work}

\textbf{Policy-gradient optimization challenges.} Vanishing gradients in policy gradients are well known, studied even in tabular MDPs \citep{JMLR:v22:19-736}. \citet{NEURIPS2020_f1cf2a08} identify a ``gravity well'' induced by the softmax parameterization, leading to poor optimization while \citet{pmlr-v134-li21b} further show that softmax PG methods can require exponentially many steps to converge. More recently, \citet{razin2024vanishing} studied vanishing gradients in LLM RL finetuning showing that the gradient becomes small when the reward variance of the policy is small. Our analysis complements these results by studying a specific case with 3 outcomes where standard PG methods can get trapped near saturated outcomes. We further show that the \textsc{GracefulRL} objective avoids this pathology, retaining a non-vanishing optimization signal near such saturated solutions.

\textbf{Likelihood-based objectives for RL.} Recent work has explored replacing expected-reward objectives with likelihood-based objectives exemplified by MaxRL \citep{tajwar2026maximum} and TailRL \citep{ramasubramanian2026taillikelihoodreinforcementlearning}. They show that these objectives admit an interpretation as a harmonically weighted combination of \(\operatorname{Best@}k\) objectives. We arrive at the same solution through an optimization perspective. We analyse how the logarithmic objective avoids gradient suppression near undesirable saturation which prevents a failure mode in expected reward maximization.

\textbf{Uncertainty quantification for LLMs.} 
Selective classification studies the tradeoff between prediction risk and coverage \citep{1054406,JMLR:v11:el-yaniv10a,pmlr-v97-geifman19a,NEURIPS2019_0c4b1eeb}. Recent work on LLM uncertainty and abstention has constructed specialized data to train refusal behavior \citep{zhang-etal-2024-r,zhao2026gracereinforcementlearninggrounded}, elicited verbalized confidence \citep{kadavath2022languagemodelsmostlyknow,lin2022teaching,tian-etal-2023-just,xiong2024can}, estimated uncertainty through repeated sampling \citep{manakul2023selfcheckgpt,kang2025scalable,wang2023selfconsistency}, combined multiple models at inference time \citep{stengel-eskin2024lacie}, directly trained models for uncertainty quantification \citep{damani2025beyond,leng2025taming,chandak2026scaling}, or introduced inference-time stopping rules that detect when continued reasoning is unlikely to help \citep{davidov2026knowing}. Our goal however is different: rather than estimating uncertainty, we aim to \emph{learn how the policy should act under uncertainty}. This distinction becomes important in sequential decision-making settings, where uncertainty quantification is not enough: the fallback policy execution must also be a part of the policy itself and its learnt end-to-end. Moreover, our approach doesn't require specialized datasets, is not LLM-specific, and naturally extends to multi-step interactions.

\textbf{Learning abstention with RL.}
Some prior works try to use RL for abstention. For instance, \citet{xu2024rejectionimprovesreliabilitytraining} use knowledge feedback to identify questions outside a model's knowledge boundary and reward rejection on these examples. \cite{wei2026truthrl}, \cite{pan2026tiartrajectoryinformedadvantagereweighting}, \cite{ma2026karlknowledgeawarereasoningreinforcement}, \cite{wu2026mitigating} and \cite{zhai-etal-2026-abstain} directly train for abstention using policy gradient style methods such as GRPO and PPO with and without heuristics such as reward tuning or modifying advantages. However, we show that this procedure can be brittle depending on the model and the task. We show that our approach is more robust and avoids issues associated with PG methods in a principled manner. Moreover, our formulation treats fallback as a general behavior rather than an LLM-specific refusal, allowing it to extend to sequential and non-language decision-making settings.

\textbf{Safe and constrained reinforcement learning.}
Our work is also related to safe RL, which studies improving task performance while satisfying safety requirements or constraints \citep{JMLR:v16:garcia15a,pmlr-v70-achiam17a}. Other approaches intervene when an unsafe action is likely; for example, Recovery RL learns a separate recovery policy that takes control near constraint-violating regions \citep{thananjeyan2021recoveryrlsafereinforcement}. In contrast, we jointly learn task success and graceful fallback, without a secondary recovery policy intervening.

\section{Conclusion and Limitations}

 We formulated \textit{graceful failure} as an RL problem in which an agent must learn when to pursue task success and when to fall back safely. We showed that naively optimizing the expected reward can be brittle due to the underlying optimization dynamics. To address it, we introduced \textsc{GracefulRL} which provides a suitable landscape for optimization to succeed. Across the settings studied in this work, \textsc{GracefulRL} learns favorable tradeoffs between task success and error, demonstrating that fallback behavior can be learned directly as part of the policy. However, our method is designed for online RL, and we have not studied its behavior under off-policy training. We also do not investigate settings involving coding agents or long horizon agentic training due to constraints on compute. Finally, specifying appropriate fallback behavior can be difficult and may differ substantially across domains. Defining \textit{graceful failure} across different domains, deriving scaling laws for it, and evaluating its generalization to unseen tasks remains an important direction for future research.

\section{Acknowledgements}

We would like to thank Aarti Singh, Sachin Goyal, Pratyush Maini, Samuel Sokota, Jatin Prakash, Matan Shtepel, Zico Kolter, Gaurav Ghoshal and Bhavya Agrawalla for their feedback on the project and suggestions for improving the manuscript. DA is especially grateful to Yiding Jiang for their guidance on approach to research and for suggesting experimental choices that helped shape the work. This work used GH200 GPUs at DeltaAI through allocation CIS261151 and CIS250527 from the Advanced Cyberinfrastructure Coordination Ecosystem: Services \& Support (ACCESS) program, which is supported by U.S. National Science Foundation grants \#2138259, \#2138286, \#2138307, \#2137603, and \#2138296. The authors also thank the CMU FLAME center and the CMU Babel Compute Cluster for additional compute support during the project. DA is supported by the National Science Foundation under Grants CCF-2106778 and DMS-2134080. YS thanks the support of the Two Sigma
Fellowship. FT gratefully acknowledges the support of Qualcomm Innovation
Fellowship and Bosch Research and Technology Center.

\bibliography{iclr2027_conference}

\appendix
\section{Proof of the Results in Section~\ref{sec:collapse}}
\label{app:collapse}

\begin{enumerate}
\item First, we fix terminology and show some simple identities which relate outcome probabilities to conditional score functions in  Section~\ref{app:setup}.    
\item Second, we show that under a boundedness assumption on the logit gradients, a standard regularity condition, the conditional score functions are bounded in Section~\ref{app:bounded-score-function}.
\item Then, we show that this causes the gradient of the expected reward objective to vanish when the policy approaches saturation in Section~\ref{app:reinforce-vanishing},
\item Finally, we show that under a non-degeneracy condition, the gradient of the log-barrier based objective does not vanish near sub-optimal saturation in Section~\ref{app:tailrl-nonvanishing}.
\end{enumerate}

It is worth noting that our proofs are only for the population level gradients and not for stochastic estimators which we leave for future work.
\subsection{Setup}
\label{app:setup}

Throughout we fix a prompt $x$, a fallback reward of $c\in (0, 1)$, and write $\mathcal T_o$ for the set of
trajectories whose outcome is $o\in\{S,F,E\}$ (\textsc{Success}, \textsc{Fallback}, and \textsc{Error}
respectively), and denote $p_o(\theta)$ for its probability under $\pi_\theta(\cdot\mid x)$. Next, we define the \textit{conditional score function}.
\[
    g_o(\theta)
    :=
    \mathbb E_{\tau\sim\pi_\theta(\cdot\mid x)}
    \left[
        \nabla_\theta\log\pi_\theta(\tau\mid x)
        \;\middle|\;
        \tau\in\mathcal T_o
    \right]
\] 
Differentiating $p_o=\sum_{\tau\in\mathcal T_o}\pi_\theta(\tau\mid x)$ and using
$\nabla_\theta\pi_\theta=\pi_\theta\nabla_\theta\log\pi_\theta$ gives
\begin{equation}
    \nabla_\theta p_o(\theta)
    =
    p_o(\theta)g_o(\theta),
    \qquad o\in\{S,F,E\}.
    \label{eq:conditional-score-identity}
\end{equation}
It is also worth noting the following identity:
\begin{equation}
\begin{aligned}
\nabla_\theta(p_S+p_F+p_E)
&= \nabla_\theta(1) = 0, \\
\Rightarrow\quad
p_S g_S + p_F g_F + p_E g_E
&= 0.
\end{aligned}
\label{eq:conditional-score-zero}
\end{equation}

Identity ~\ref{eq:conditional-score-identity} shows that the gradient of an outcome probability is the
\emph{probability} of outcome $p_Y$ times a \emph{direction} $g_Y$. 

Also, we define the softmax parameterization. We consider policies parameterized by softmax over a network producing logits
$z_\theta(s_t)\in\mathbb{R}^{|\mathcal A|}$:
\[
\pi_\theta(a_t\mid s_t)
=
\frac{\exp(z_\theta(s_t)_{a_t})}
{\sum_{a'\in\mathcal A}\exp(z_\theta(s_t)_{a'})}
\]
Autoregressive language models are a special case where $\mathcal A$ is the
vocabulary of the model and $s_t=(x,a_{<t})$ is the first $t$ generated tokens along with the prompt. Next, we show that a common assumption of boundedness of logit gradients leads to boundedness of the conditional score function under the softmax parameterization.

\subsection{Bounded Conditional Scores}
\label{app:bounded-score-function}
\begin{assumption}[Bounded logit gradients]
\label{ass:bounded-logit-gradients}
There exists a constant $B < \infty$ such that, for every
$\theta\in\Theta$, every state $s$, and every discrete action $a\in\mathcal A$,
\[
    \left\|
        \nabla_\theta z_\theta(s)_a
    \right\|
    \leq B.
\]
\end{assumption}

\begin{lemma}[Bounded conditional scores]
\label{lem:score-bound}
Under Assumption~\ref{ass:bounded-logit-gradients}, and writing $G:=2HB$,
\begin{equation}
    \|g_o(\theta)\|
    \leq G,
    \qquad
    o\in\{S,F,E\},
    \label{eq:conditional-score-upper-bound}
\end{equation}
for every $\theta\in\Theta$.
\end{lemma}

\begin{proof}
For the softmax policy,
\[
    \nabla_\theta \log \pi_\theta(a\mid s)
    =
    \nabla_\theta z_\theta(s)_a
    -
    \sum_{a'\in\mathcal A}
    \pi_\theta(a'\mid s)\nabla_\theta z_\theta(s)_{a'} .
\]
The second term is a convex combination of vectors of norm at most $B$, so it has norm at most $2B$
\[
    \Rightarrow\left\|
        \nabla_\theta \log \pi_\theta(a\mid s)
    \right\|
    \leq 2B.
\]
Summing over an entire trajectory $\tau$ which is at most $H$ steps,
\[
    \left\|
        \nabla_\theta \log \pi_\theta(\tau\mid x)
    \right\|
    \leq 2HB
    =
    G .
\]
Each $g_Y$ is a conditional expectation of these trajectory scores, hence a
convex combination of vectors of norm at most $G$, which
gives~\eqref{eq:conditional-score-upper-bound}.
\end{proof}

\subsection{Policy Gradient Vanishes Near Saturation}
\label{app:reinforce-vanishing}

First, we express the policy gradients in terms of the conditional scores. The objective function for policy gradient is the expected reward $p_S+cp_F$. Using 
identity~\ref{eq:conditional-score-identity} we get
\begin{equation}
\begin{aligned}
    \mathcal J_{\mathrm{pg}}(x)
    &= p_S + c p_F,\\
    \Rightarrow \nabla_\theta \mathcal J_{\mathrm{pg}}(x)
    &= p_S\,g_S + c\,p_F\,g_F
\end{aligned}
\label{eq:reinforce-gradient}
\end{equation}

\begin{proposition}[Gradient suppression under fallback saturation]
\label{prop:reinforce-vanishing}
Suppose Assumption~\ref{ass:bounded-logit-gradients} holds. Then for every $\theta\in\Theta$, and every outcome $o \in \{\textsc{Success}, \textsc{Fallback}, \textsc{Error}\}$, as the policy saturates on any outcome, that is, when $p_o\rightarrow1$
\[
    \left\|
        \nabla_\theta \mathcal{J}_{\mathrm{pg}}(\theta)
    \right\|
    \le
    G\bigl(1-p_o(x)\bigr),
\]
\end{proposition}
\begin{proof}
Under Assumption~\ref{ass:bounded-logit-gradients}, the conditional scores satisfy $\|g_S\|,\|g_F\|,\|g_E\|\le G.$ We consider the three possible saturated outcomes separately.

If $o=\textsc{Error}$, then from Eq.~\ref{eq:reinforce-gradient},
$$
\begin{aligned}
\|\nabla_\theta \mathcal{J}_{\mathrm{pg}}\|
&=\|p_Sg_S+c p_Fg_F\|\\
&\le G(p_S+c p_F)\\
&\le G(p_S+p_F)
=G(1-p_E).
\end{aligned}
$$

If $o=\textsc{Fallback}$, using $p_Fg_F=-p_Sg_S-p_Eg_E$ from \eqref{eq:conditional-score-zero} gives $ \nabla_\theta \mathcal{J}_{\mathrm{pg}} =(1-c)p_Sg_S-cp_Eg_E. $ Hence,
$$
\begin{aligned}
\|\nabla_\theta \mathcal{J}_{\mathrm{pg}}\|
&\le G\bigl((1-c)p_S+c p_E\bigr)\\
&\le G(p_S+p_E)
=G(1-p_F).
\end{aligned}
$$

Finally, if $o=\textsc{Success}$, using $p_Sg_S=-p_Fg_F-p_Eg_E$ from \eqref{eq:conditional-score-zero} gives $ \nabla_\theta \mathcal{J}_{\mathrm{pg}} =-(1-c)p_Fg_F-p_Eg_E,$ and therefore
$$
\begin{aligned}
\|\nabla_\theta \mathcal{J}_{\mathrm{pg}}\|
&\le G\bigl((1-c)p_F+p_E\bigr)\\
&\le G(p_F+p_E)
=G(1-p_S).
\end{aligned}
$$

Thus, for any $o\in\{{\textsc{Success},\textsc{Fallback},\textsc{Error}}\}$, $ \|\nabla_\theta \mathcal{J}_{\mathrm{pg}}\| \le G(1-p_o), $ which goes to zero as $p_o\to1$.
\end{proof}

\subsection{Gradient of the log-barrier objective does not vanish}
\label{app:tailrl-nonvanishing}

\begin{assumption}[Non-degeneracy]
\label{ass:reachability}
There exist $\varepsilon,\delta_F,\delta_E>0$ such that for all $\theta\in\Theta$:
\[p_F(\theta)\ge 1-\varepsilon \Rightarrow \|g_S(\theta)\|\ge\delta_F,\] and
\[p_E(\theta)\ge 1-\varepsilon \Rightarrow \inf_{\lambda\in[1-c,1]}\|\lambda g_S+(1-\lambda)g_F\|\ge\delta_E\].
\label{ass:reachability-error}
\end{assumption}

Intuitively, this assumption requires that, near a degenerate policy, there remains some parameter direction along which probability mass can be moved toward a better outcome. When the policy is almost entirely concentrated on \textsc{Fallback}, the success gradient must remain nonzero, so the model can still move probability mass toward \textsc{Success}. Similarly, when the policy is almost entirely concentrated on \textsc{Error}, the gradients associated with \textsc{Success} and \textsc{Fallback} must not cancel for any weighting induced by the objective.

Note that Assumption~\ref{ass:reachability} does not change the implication of Proposition~\ref{prop:reinforce-vanishing}. That result is an upper bound on the policy-gradient norm and follows from Assumption~\ref{ass:bounded-logit-gradients} alone. The non-degeneracy condition is needed here only to obtain a lower bound on the log-barrier gradient.

First, we explain that this assumption is a reasonable one to make. Assumption~\ref{ass:reachability} constrains the relevant score directions only near the two undesirable saturated policies. A single-step softmax policy illustrates why these directions need not vanish even when the corresponding outcomes are rare. Consider three possible actions, one for each outcome, with trainable logits $z=(z_S,z_F,z_E)$ and probabilities $p=\operatorname{softmax}(z)$. Since each outcome identifies a single action, its conditional score is
\[
    g_Y=\nabla_z\log p_Y=e_Y-p,
    \qquad Y\in\{S,F,E\},
\]
where $e_Y$ is the corresponding standard basis vector. Near fallback saturation, $p\to e_F$. Therefore,
\[
    g_S=e_S-p\longrightarrow e_S-e_F,
    \qquad \|g_S\|\longrightarrow\sqrt{2}.
\]
Thus, the score conditioned on success remains bounded away from zero even as success becomes arbitrarily unlikely. Similarly, near error saturation, $p\to e_E$. Then, for every $\lambda\in[0,1]$, a convex combination of the success and fallback scores satisfies
\[
    \lambda g_S+(1-\lambda)g_F
    \longrightarrow
    \lambda e_S+(1-\lambda)e_F-e_E.
\]
The limiting squared norm for this is $\lambda^2+(1-\lambda)^2+1$,
whose minimum over $\lambda\in[0,1]$ is $3/2$. Hence the second direction is also bounded away from zero. By
continuity, positive constants $\delta_F$ and $\delta_E$ therefore exist in sufficiently small neighborhoods of the two saturated policies. For a general parameterization, Assumption~\ref{ass:reachability} requires that the logit
Jacobian not annihilate these same directions.

\begin{proposition}[The log-barrier objective retains a non-vanishing signal]
\label{prop:tailrl-nonvanishing}
Suppose Assumptions~\ref{ass:bounded-logit-gradients}
and~\ref{ass:reachability} hold, and let $c\in(0,1)$. For either suboptimal
outcome $o\in\{F,E\}$, as the policy saturates on that outcome, i.e., as
$p_o(x)\to 1$,
\[
    \left\|
        \nabla_\theta \mathcal{J}_{log}(\theta)
    \right\|
    =
    \Theta(1).
\]
\end{proposition}

\begin{proof}

We have 
\begin{equation}
\begin{aligned}
    \mathcal{J}_{log} 
    &= (1-c)\log p_s + c \log(1-p_E)\\
    \Rightarrow \nabla_\theta \mathcal{J}_{log}&= (1-c)g_S-c\frac{p_E}{1-p_E}g_E.
\end{aligned}
    \label{eq:tailrl-gradient}
\end{equation}

We consider the two saturated outcomes separately. First, suppose
$p_F\to1$. The second expression in~\eqref{eq:tailrl-gradient} shows that the
logarithm has cancelled the factor $p_S$, leaving the success direction with a
constant coefficient. For $p_F$ sufficiently close to one, using the reverse triangle inequality,
\begin{align}
    \left\|
        \nabla_\theta \mathcal{J}_{log}
    \right\|
    &\geq
    (1-c)\|g_S\|
    -
    c\frac{p_E}{1-p_E}\|g_E\|
    \nonumber\\
    &\geq
    (1-c)\delta_F
    -
    cG\frac{p_E}{1-p_E},
    \label{eq:tail-gradient-lower}
\end{align}
but also by the triangle inequality, 
\begin{align}
    \left\|
        \nabla_\theta \mathcal{J}_{log}
    \right\|
    &\leq
    (1-c)G
    +
    cG\frac{p_E}{1-p_E}.
    \label{eq:tail-gradient-upper}
\end{align}
Since $p_E/(1-p_E)\to0$ when $p_F\rightarrow1$ these bounds imply
\[
    \liminf_{p_F\to1}\|\nabla_\theta\mathcal J_{\log}\| \ \ge\ (1-c)\,\delta_F \ >\ 0,
\qquad
\limsup_{p_F\to1}\|\nabla_\theta\mathcal J_{\log}\| \ \le\ (1-c)\,G.
\]

Now, for the case when error is saturating, using identity \ref{eq:conditional-score-zero}, we get,
\begin{align}
\nabla_\theta \mathcal{J}_{log} 
&= (1-c)g_S -c\frac{p_E}{1-p_E}g_E\\
&= (1-c)g_S +c\frac{p_Sg_S + p_Fg_F}{p_S+p_F}\\
&= (1-c+c\frac{p_S}{p_S+p_F})g_S + c\frac{p_F}{p_F+p_S}g_F
\end{align}

 To simplifying some notation, we define:
\[
    \lambda_\theta=\frac{p_S}{p_S+p_F}\in[0,1],
    \qquad
    \widetilde\lambda_\theta=(1-c)+c\lambda_\theta\in[1-c,1].
\]

The first expression in~\eqref{eq:tailrl-gradient} becomes
\[
    \nabla_\theta \mathcal{J}_{log}
    =\widetilde\lambda_\theta g_S
    +(1-\widetilde\lambda_\theta)g_F.
\]

For $p_E$ sufficiently close to one, Assumption~\ref{ass:reachability} and
Lemma~\ref{lem:score-bound} therefore give
\[
    \delta_E
    \le
    \left\|\nabla_\theta \mathcal{J}_{log}\right\|
    \le
    \widetilde\lambda_\theta\|g_S\|
    +(1-\widetilde\lambda_\theta)\|g_F\|
    \le G.
\]
Thus the gradient is bounded above and away from zero near either suboptimal
saturated outcome. Hence, for $o\in\{F,E\}$,
\[
    \boxed{
    \left\|
        \nabla_\theta \mathcal{J}_{log}
    \right\|
    =
    \Theta(1)
    \quad\text{as }p_o\to1.
    }
\]
\end{proof}

\clearpage
\section{Effects of the \texorpdfstring{$\eps$}{epsilon} parameter}
\label{app:eps_theory}

The reciprocal weights in Algorithm~\ref{alg:graceful-rl} can have high variance when success or non-error trajectories are rare. At fixed $N$, small $\epsilon$ keeps the expected update close to the unsmoothed finite-rollout target; as $N$ grows, its rescaled gradient approaches $\nabla_\theta\mathcal J_{\log}$. Larger $\epsilon$ shrinks and equalizes the denominators, reducing variance while introducing bias relative to the $\epsilon=0$ update.

First, we prove that the bias is linear in the smoothing factor $\eps$. Let us define $\hat{g}(\eps)$ to be the stochastic gradient obtained from using Algorithm \ref{alg:graceful-rl}. Let us define its expected value as $g(\eps)$. Then, we show the following:

\begin{proposition} The bias introduced by smoothing, is linear in $\eps$
\label{prop:linear_bias}
\begin{align}
    \|g(\eps) - g(0)\| = O(\eps)
\end{align}
\end{proposition}

First, we prove a lemma.

\begin{lemma}
Let $N_S,N_F,N_E$ denote the numbers of \textsc{success}, \textsc{fallback}, and \textsc{error} samples among $N$ rollouts, where $N=N_S+N_F+N_E$. Let $a_{i,\epsilon}$ denote the centered advantage used by Algorithm~\ref{alg:graceful-rl} with smoothing parameter $\epsilon\geq 0$. Then, for every rollout $i$,
\[
a_{i,\epsilon}-a_{i,0}=O(\epsilon)
\qquad\text{as }\epsilon\to 0.
\]
Note that $a_{i,\eps}$ is a random variable where the randomness is coming from the rollouts sampled from the policy, conditioned on $N_S, N_F,$ and $N_E$.
\end{lemma}

\begin{proof}
The centered advantage takes one of three forms depending on the outcome of
rollout $i$:
\[
a_{i,\epsilon}
=
\begin{cases}
\displaystyle
\frac{(1-c)(N_F+N_E)}
     {N(N_S+\epsilon)}
+
\frac{cN_E}
     {N(N_S+N_F+\epsilon)},
& o_i=\textsc{success},
\\[1.2em]
\displaystyle
-\frac{(1-c)N_S}
      {N(N_S+\epsilon)}
+
\frac{cN_E}
     {N(N_S+N_F+\epsilon)},
& o_i=\textsc{fallback},
\\[1.2em]
\displaystyle
-\frac{(1-c)N_S}
      {N(N_S+\epsilon)}
-
\frac{c(N_S+N_F)}
      {N(N_S+N_F+\epsilon)},
& o_i=\textsc{error}.
\end{cases}
\]

Note that in the ratios above, a ratio is understood as 0 when it is of the $0/0$ form. 

Then, for any integer $m \ge 1$, $
\frac{1}{m+\epsilon} - \frac{1}{m}
=
-\frac{\epsilon}{m(m+\epsilon)}
=
O(\epsilon).
$ Therefore, every nonzero term appearing in $a_{i,\epsilon}$ changes by $O(\epsilon)$ when $\epsilon$ is perturbed from $0$ to $\epsilon$. Since each centered advantage contains only finitely many such terms, $a_{i,\epsilon} - a_{i,0} = O(\epsilon)$.

Zero-count cases can also be handled this way. If $N_S=0$, all terms proportional to $N_S/(N_S+\epsilon)$ vanish. Likewise, if $N_S+N_F=0$, all terms proportional to $(N_S+N_F)/(N_S+N_F+\epsilon)$ vanish. Thus, for a fixed count configuration $(N_S, N_F, N_E)$, we have $a_{i,\epsilon} - a_{i,0} = O(\epsilon)$.
\end{proof}

\begin{proof}
Now, we prove Proposition \ref{prop:linear_bias}. First let us define the following conditional expectation:
\[
g_\eps(N_S, N_F, N_E) := \mathbb{E}\left[
\frac{1}{N}\sum_{i=1}^N
a_{i,\epsilon}s_i
\;\middle|\;N_S, N_F, N_E\right].
\]
From Assumption~\ref{ass:bounded-logit-gradients}, we have that $\|s_i\|\leq G,$, then for any fixed configuration of counts $(N_S, N_F, N_E)$ we obtain:
\begin{align*}
\|g_\epsilon(N_S, N_F, N_E)-g_0(N_S, N_F, N_E)\|
&\leq
\mathbb{E}\left[
\frac{1}{N}\sum_{i=1}^N
\left|a_{i,\epsilon}-a_{i,0}\right|
\|s_i\|\;\middle|\; N_S, N_F, N_E
\right] \\
&\leq
G\cdot
\mathbb{E}\left[
\frac{1}{N}\sum_{i=1}^N
\left|a_{i,\epsilon}-a_{i,0}\right|\;\middle|\;N_S, N_F, N_E
\right]
=O(\epsilon).
\end{align*}

Note that for fixed $N$, there are only finitely many possible count configurations
$(N_S,N_F,N_E)$. Therefore, by the law of total expectation,

\begin{align}
\|g_\epsilon-g_0\|
&=
\left\|\sum_{\substack{N_S,N_F,N_E\\N_S+N_F+N_E=N}}
\mathbb{P}(N_S,N_F,N_E)\,
[
g_\epsilon(N_S,N_F,N_E)-g_0(N_S,N_F,N_E)]\right\|\\
&\leq \sum_{\substack{N_S,N_F,N_E\\N_S+N_F+N_E=N}}\mathbb{P}(N_S,N_F,N_E)\left\|g_\epsilon(N_S,N_F,N_E)-g_0(N_S,N_F,N_E)\right\|\\
&= O(\epsilon).
\end{align}

where the last step follows because each conditional term is
$O(\epsilon)$ and there are only finitely many count configurations. Thus, the smoothing parameter \(\epsilon\) introduces only an \(O(\epsilon)\) bias in the expected gradient.
\end{proof}

If $V(\eps)$ denotes the variance of the estimator as a function of $\eps$ and $V'(0^+)<0$, then for sufficiently small $\eps>0$,
\[
V(\eps)=V(0)+V'(0^+)\eps+o(\eps)<V(0).
\]

Since Proposition~\ref{prop:linear_bias} gives $|g(\eps)-g(0)|=\mathcal{O}(\eps)$, the squared bias is only $\mathcal{O}(\eps^2)$. Therefore, for sufficiently small $\eps>0$, the linear decrease in variance dominates the quadratic increase in squared bias, yielding lower MSE. Empirically, we also find that using a small nonzero $\eps$ improves performance when training using smaller amounts of compute budget.

\clearpage
\section{SafeMaze Extended Results}
\label{app:safemaze}

\subsection{\textsc{SafeMaze} details}
For the \textsc{SafeMaze} environment, we generate a $15\times 15$ grid and pretrain the policy using a mixture of \textsc{Goal}-seeking and \textsc{Safe}-square-seeking trajectories. We ensure that the \textsc{Goal} is always reachable in the grid within the time limit.

\subsection{RLOO with a Fixed Fallback Reward}
\label{app:safemaze-fixed-rloo}

Figure~\ref{fig:safemaze-fixed-rloo} shows why selecting a fixed reward for fallback when using policy gradient is difficult. With $c=0.1$, the fallback probability tends to zero, so the policy continues attempting the goal but incurs a high error rate. At $c=0.2$, close to the policy's initial pass rate, RLOO retains nonzero success, but its performance plateaus while fallback remains the dominant outcome. For every tested reward $c\geq0.3$, the policy instead converges rapidly to the fallback-only solution, with $p_F\to1$ and $p_S\to0$. Thus, a fixed reward gives useful behavior only in a narrow range and does not reliably support continued improvement. This motivates the online reward-tuning baseline, which sets the fallback reward relative to the policy's current success rate rather than choosing a single value for the entire run.

\begin{figure}[t]
    \centering
    \includegraphics[width=0.85\textwidth]{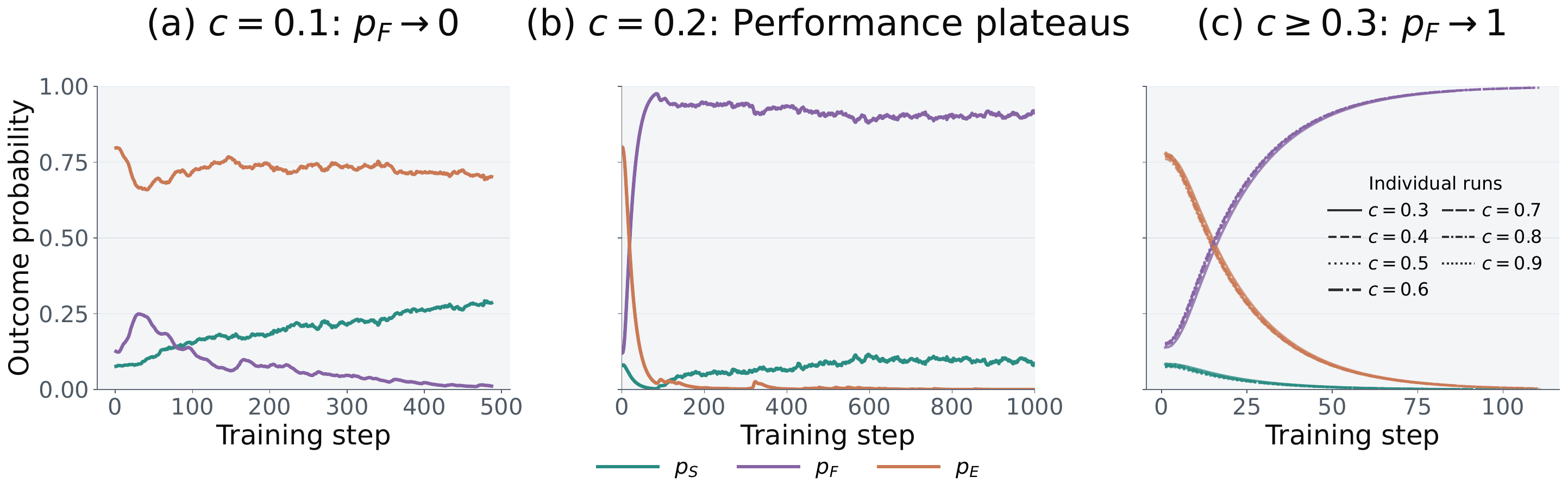}
    \caption{\textbf{RLOO with fixed fallback rewards on SafeMaze.} We measure $p_S, p_F, p_E$ during training for $c=0.1$, $c=0.2$, and $c\geq0.3$. Empirically, we observe that when $c \ll p_S$, the policy can lead to no coverage on fallback trajectories, that is, $p_F\rightarrow 0$. Similarly, when $c \gg p_S$, the policy can lead to a degenerate policy where it always goes to the \textsc{Safe} square.}
    \label{fig:safemaze-fixed-rloo}
\end{figure}

\subsection{Hyperparameters for the Main SafeMaze Comparison}
\label{app:safemaze-main-hyperparameters}

For every run in Figure~\ref{fig:safemaze-summary} \textbf{(c)}, we use $32$ environments and $64$ rollouts and train for $20{,}000$ optimization steps. For our method, we sweep the fallback reward $c\in\{0,0.2,0.4,0.6,0.8\}$. For DPG, we set $\eta=0.9$ and use the same safe-reward sweep. For the online reward-tuning baseline, we sweep the multiplier $\alpha\in\{1,0.8,0.6,0.4,0.2,0\}$ in $c_t=\alpha\widehat p_{S,t}$.

\subsection{Varying Rollout Counts and Safe Rewards}
\label{app:safemaze-rollouts}

We jointly sweep the number of rollouts per update $N\in\{16,32,64\}$ and the reward for reaching \textsc{Safe}, $c\in\{0,0.2,0.4,0.6,0.8\}$. Figure~\ref{fig:safemaze-ablations}\textbf{(a)} plots the resulting reliability frontiers after $5{,}000$ training steps. We observe that increasing $N$ improves the reliability frontier. This can be explained by the fact that larger rollout counts provide a higher-order truncation of the log objective and are more likely to include rare successful trajectories, producing a more informative update.

The effect of rollout count becomes more pronounced as the fallback reward increases. At $c=0.8$, the $16$- and $32$-rollout runs collapse to the fallback-only solution, whereas $64$ rollouts retain nonzero success. Thus, settings that place greater value on reliability also require more rollouts to approximate a higher order approximation of the objective.

\subsection{Smoothing-Parameter \texorpdfstring{$\epsilon$}{epsilon} effects}
\label{app:safemaze-epsilon}

As discussed in Section~\ref{app:eps_theory}, $\eps$ smoothing can lower variance at the cost of some bias. Figure~\ref{fig:safemaze-ablations}\textbf{(b)} shows the corresponding optimization tradeoff: larger smoothing factors improve reward earlier in training, while smaller values become preferable with more compute.

\begin{figure}[t]
    \centering
    \includegraphics[width=0.9\textwidth]{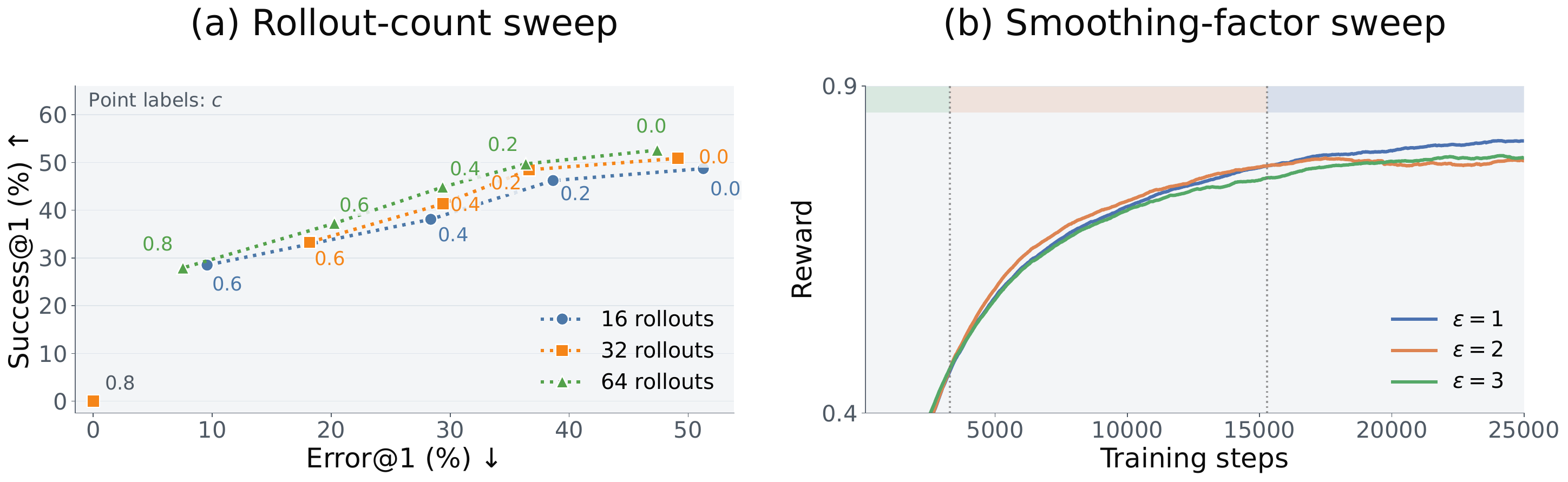}
    \caption{\textbf{SafeMaze rollout and smoothing sweeps.} \textbf{(a)} Reliability frontiers after $5{,}000$ steps for $16$, $32$, and $64$ rollouts. \textbf{(b)} Training reward with $32$ rollouts and $c=0.4$: larger $\epsilon$ helps initially, whereas $\epsilon=1$ performs best later. The colored bands mark the highest-reward curve in each training interval pointing to the general principle: when bounded by compute, larger $\eps$ helps in achieveing higher reward.}
    \label{fig:safemaze-ablations}
\end{figure}

\begin{figure}[h]
    \centering
    \includegraphics[width=0.9\textwidth]{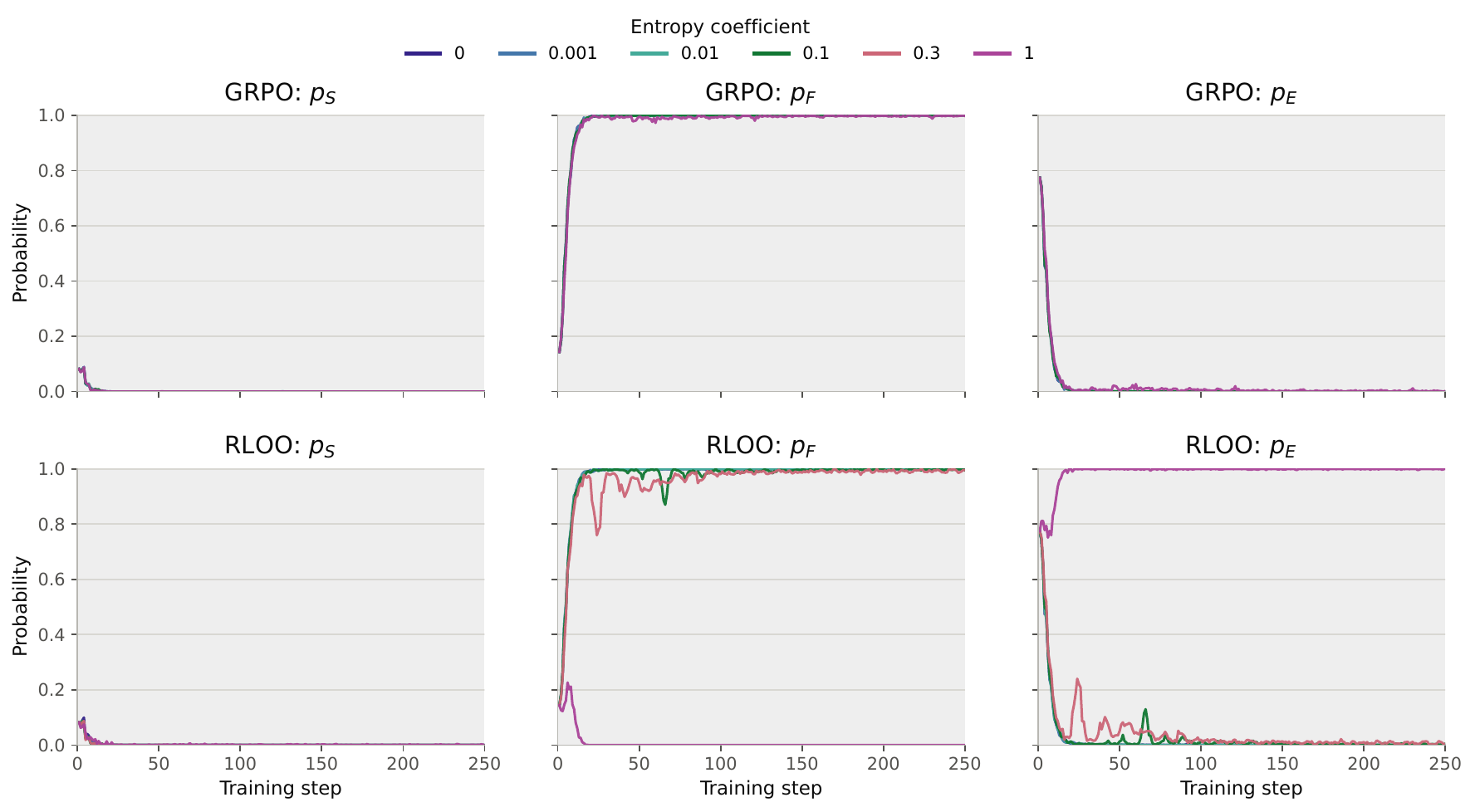}
    \caption{Entropy-bonus sweep for GRPO (top) and RLOO (bottom) on
    \textsc{SafeMaze}. Columns show the probabilities of success ($p_S$),
    fallback ($p_F$), and error ($p_E$) during training. None of the tested
    entropy coefficients preserves task success: GRPO and most RLOO runs
    collapse to fallback, while RLOO with coefficient $1$ collapses to error.}
    \label{fig:safemaze-entropy-sweep}
\end{figure}

\subsection{Entropy-Bonus Tuning}
\label{app:safemaze-entropy}

We also sweep the entropy-bonus strength for GRPO and RLOO while fixing the safe reward to $c=0.4$ and holding everything else fixed. Entropy regularization does not yield a meaningful policy in this regime.  Figure~\ref{fig:safemaze-entropy-sweep} shows that every setting drives $p_S$ to zero within a few optimization steps. GRPO converges to the fallback-only policy across the entire sweep. RLOO does the same for entropy coefficients up to $0.3$, while the largest coefficient, $1$, instead converges to $p_E = 1$. The entropy bonus encourages stochasticity over actions, but it still suffers from the vanishing gradient problem near a saturated policy \citep{JMLR:v22:19-736,pmlr-v97-ahmed19a}.

\clearpage
\section{Mathematical-Reasoning Training Details and Validation}
\label{app:math-training}

\subsection{Training Hyperparameters}
\label{app:math-hyperparameters}

For training, we use the verl \citep{sheng2024hybridflow, 10.1145/3689031.3696075} framework. Tables \ref{tab:math-training-hparams-4B} and~\ref{tab:math-training-hparams-1.7B} give the complete training configuration for the comparison of Qwen3-4B-Base and Qwen3-1.7B-Base. We match all shared optimization and sampling hyperparameters across methods. RLCR uses RLOO with the calibration reward defined in Section~\ref{sec:math-reasoning}, whereas our method uses the estimator in Algorithm \ref{alg:graceful-rl} with $\epsilon=1$ and jointly trains on the five conditioning levels.

\begin{table}[H]
\centering

\small
\begin{tabularx}{\textwidth}{@{}lXX@{}}
\toprule
\textbf{Hyperparameter} & \textbf{RLCR} & \textbf{Ours} \\
\midrule
Gradient estimator & RLOO & GracefulRL \\
Training data & DAPO-Math-17K & DAPO-Math-17K \\
Fallback utilities $c$ & -- & $\{0,0.2,0.4,0.6,0.8\}$ \\
Smoothing $\epsilon$ & -- & $1$ \\
Optimizer & AdamW & AdamW \\
Learning rate / schedule & $10^{-6}$ / constant & $10^{-6}$ / constant \\
Warmup / weight decay & none / $0.01$ & none / $0.01$ \\
Prompts $\times$ rollouts per update & $32\times32$ & $32\times32$ \\
Maximum prompt / response length & $2{,}048$ / $4{,}096$ & $2{,}048$ / $4{,}096$ \\
Rollout temperature / top-$p$ & $1.0$ / $1.0$ & $1.0$ / $1.0$ \\
Policy epochs / minibatch size & $1$ / $32$ & $1$ / $32$ \\
Microbatch size per GPU & $2$ & $2$ \\
Policy clip ratio / gradient clipping & $0.2$ / $1$ & $0.2$ / $1$ \\
Entropy coefficient / KL loss & $0$ / disabled & $0$ / disabled \\
Loss aggregation & seq-mean-token-sum-norm & seq-mean-token-sum-norm \\
Precision / sharding & bfloat16 / FSDP & bfloat16 / FSDP \\
Training hardware & $2\times8$ H100GPUs & $2\times8$ H100GPUs \\
Training steps / checkpoint interval & $1{,}000$ / $200$ & $1{,}000$ / $200$ \\
Validation interval & $50$ steps & $50$ steps \\
Validation rollouts / temperature / top-$p$ & $32$ / $0.6$ / $0.95$ & $32$ / $0.6$ / $0.95$ \\
\bottomrule
\end{tabularx}
\caption{Training hyperparameters for the Qwen3-4B-Base RLCR and \textsc{GracefulRL} runs. A dash denotes a method-specific parameter that does not apply.
}
\label{tab:math-training-hparams-4B}
\end{table}

\begin{table}[H]
\centering
\small
\begin{tabularx}{\textwidth}{@{}lXX@{}}
\toprule
\textbf{Hyperparameter} & \textbf{RLCR} & \textbf{Ours} \\
\midrule
Gradient estimator & RLOO & GracefulRL \\
Training data & DAPO-Math-17K & DAPO-Math-17K \\
Fallback utilities $c$ & -- & $\{0,0.2,0.4,0.6,0.8\}$ \\
Smoothing $\epsilon$ & -- & $1$ \\
Optimizer & AdamW & AdamW \\
Learning rate / schedule & $10^{-6}$ / constant & $10^{-6}$ / constant \\
Warmup / weight decay & none / $0.01$ & none / $0.01$ \\
Prompts $\times$ rollouts per update & $32\times32$ & $32\times32$ \\
Maximum prompt / response length & $2{,}048$ / $4{,}096$ & $2{,}048$ / $4{,}096$ \\
Rollout temperature / top-$p$ & $1.0$ / $1.0$ & $1.0$ / $1.0$ \\
Policy epochs / minibatch size & $1$ / $32$ & $1$ / $32$ \\
Microbatch size per GPU & $2$ & $2$ \\
Policy clip ratio / gradient clipping & $0.2$ / $1$ & $0.2$ / $1$ \\
Entropy coefficient / KL loss & $0$ / disabled & $0$ / disabled \\
Loss aggregation & seq-mean-token-sum-norm & seq-mean-token-sum-norm \\
Precision / sharding & bfloat16 / FSDP & bfloat16 / FSDP \\
Training hardware & $1\times8$ H100GPUs & $1\times8$ H100GPUs \\
Training steps / checkpoint interval & $1{,}500$ / $300$ & $1{,}500$ / $300$ \\
Validation interval & $50$ steps & $50$ steps \\
Validation rollouts / temperature / top-$p$ & $32$ / $0.6$ / $0.95$ & $32$ / $0.6$ / $0.95$ \\
\bottomrule
\end{tabularx}
\caption{Training hyperparameters for the Qwen3-1.7B-Base RLCR and \textsc{GracefulRL} runs. A dash denotes a method-specific parameter that does not apply.
}
\label{tab:math-training-hparams-1.7B}
\end{table}
\subsubsection{RLCR Training Prompt}
\label{app:rlcr-prompt}

For the RLCR baseline, we replace \texttt{\{problem\}} with the mathematical question and use the following prompt verbatim:

\begin{minesweeperpromptbox}{RLCR prompt}
{problem}
Please reason step by step, and put your final answer in the following format:
<answer>\boxed{<final answer>}</answer>
Then, reason about how confident you are in your answer using the following scale:
- 0.1: You are guessing or have very little idea
- 0.2: You are quite unsure, the approach might be wrong
- 0.3: You are somewhat unsure, there are likely errors
- 0.4: You are slightly unsure, some steps feel shaky
- 0.5: You could go either way
- 0.6: You are slightly confident, the approach seems right
- 0.7: You are fairly confident, most steps check out
- 0.8: You are quite confident, only minor doubt remains
- 0.9: You are very confident and have verified key steps
- 1.0: You are certain the answer is correct
Report your confidence as:
<confidence>\boxed{<your score>}</confidence>
\end{minesweeperpromptbox}
Unlike the RLCR paper, we find that Qwen3 models can learn to follow this instruction after a couple of training steps and there is no need for extra SFT data to warm start the instruction following.

\subsubsection{GracefulRL Conditioning Suffixes}
\label{app:conditioning-suffixes}

We append one of the following suffixes to each mathematical problem and use the corresponding fallback utility $c$ when computing its training reward. All five variants are mixed during training.

\begin{table}[H]
\centering
\small
\begin{tabularx}{\textwidth}{@{}cX@{}}
\toprule
$c$ & \textbf{Conditioning suffix} \\
\midrule
$0.0$ & ``Be very optimistic and always attempt an answer, even if you are unsure.'' \\
$0.2$ & ``Be optimistic but stay mindful of potential mistakes.'' \\
$0.4$ & ``Be balanced, leaning towards attempting an answer.'' \\
$0.6$ & ``Be moderately careful and verify your approach before committing to an answer.'' \\
$0.8$ & ``Be quite careful and avoid guessing.'' \\
\bottomrule
\end{tabularx}
\caption{Conditioning suffixes for the five GracefulRL fallback utilities.}
\label{tab:conditioning-suffixes}
\end{table}

\begin{figure}[h]
    \centering
    \includegraphics[width=0.6\textwidth]{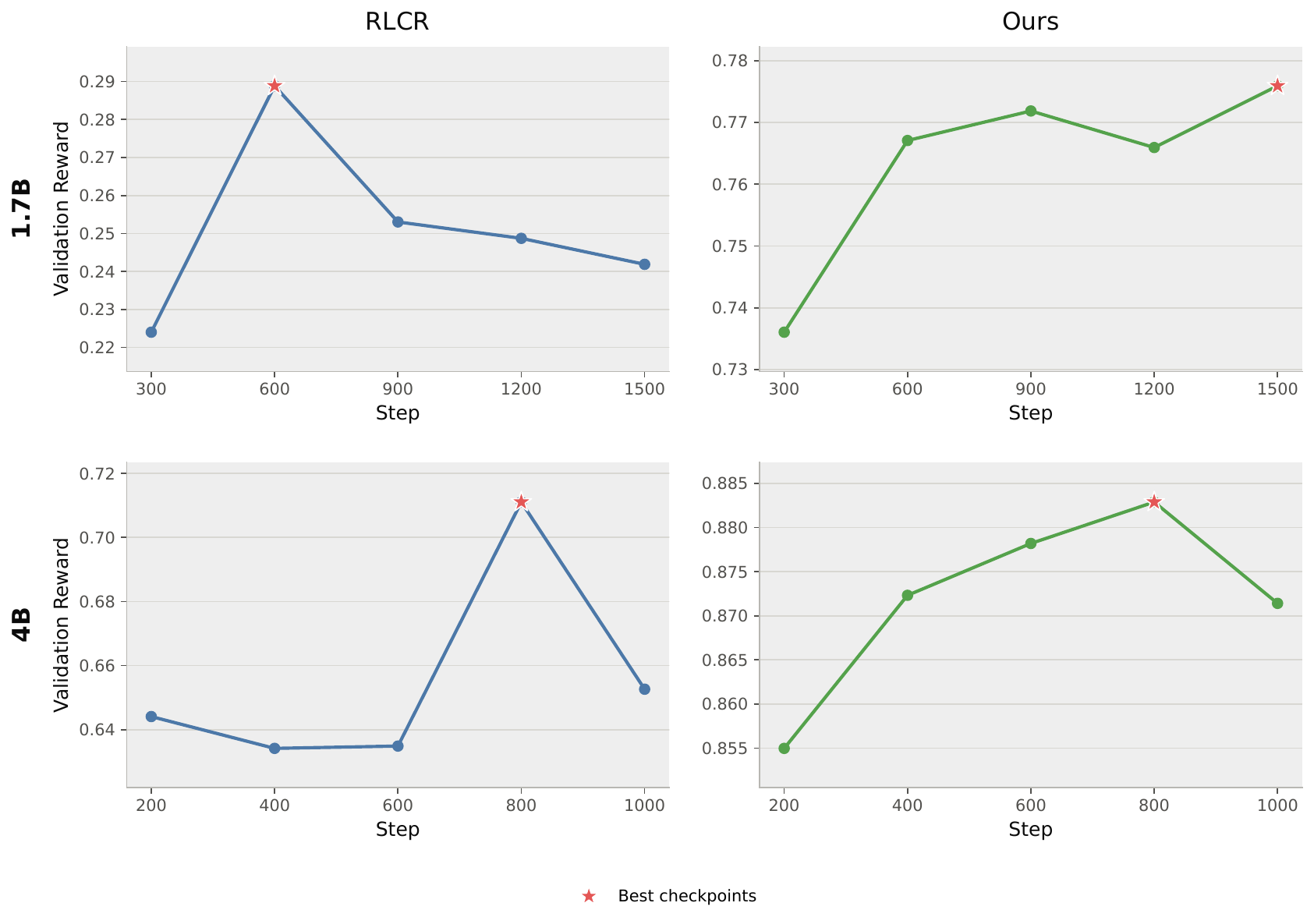}
    \caption{MATH500 validation reward across training checkpoints for RLCR
    (left) and our conditioned \textsc{GracefulRL} method (right), at the 1.7B (top) and 4B
    (bottom) model scales. Stars mark the checkpoints selected for the final
    evaluation.}
    \label{fig:math-validation-curves}
\end{figure}
\subsection{Validation and Checkpoint Selection}
\label{app:additional_lm_results}
We evaluate checkpoints on the MATH500 validation set using $32$ sampled
responses per problem. For our conditioned model, we compute the reward at each
of the five conditioning levels and average across them to construct the
validation metric. For RLCR, we use the mean Brier reward. As shown in
Figure~\ref{fig:math-validation-curves}, the selected Qwen3-1.7B-Base
checkpoints are step $1{,}500$ for our method and step $600$ for RLCR. For
Qwen3-4B-Base, both methods attain highest validation reward at step
$800$. We select these checkpoints before evaluating on the test suite.

\subsection{Validation on All Datasets}
\label{app:rlcr-per-dataset}

Figures~\ref{fig:rlcr-1.7b-datasets} and~\ref{fig:rlcr-4b-datasets} expand the
aggregate comparison in Figure~\ref{fig:rlcr-aggregate} into separate reliability frontiers for AIME, AMC, JEEBench and MATH500.

\begin{figure}[t]
    \centering
    \includegraphics[width=\textwidth]{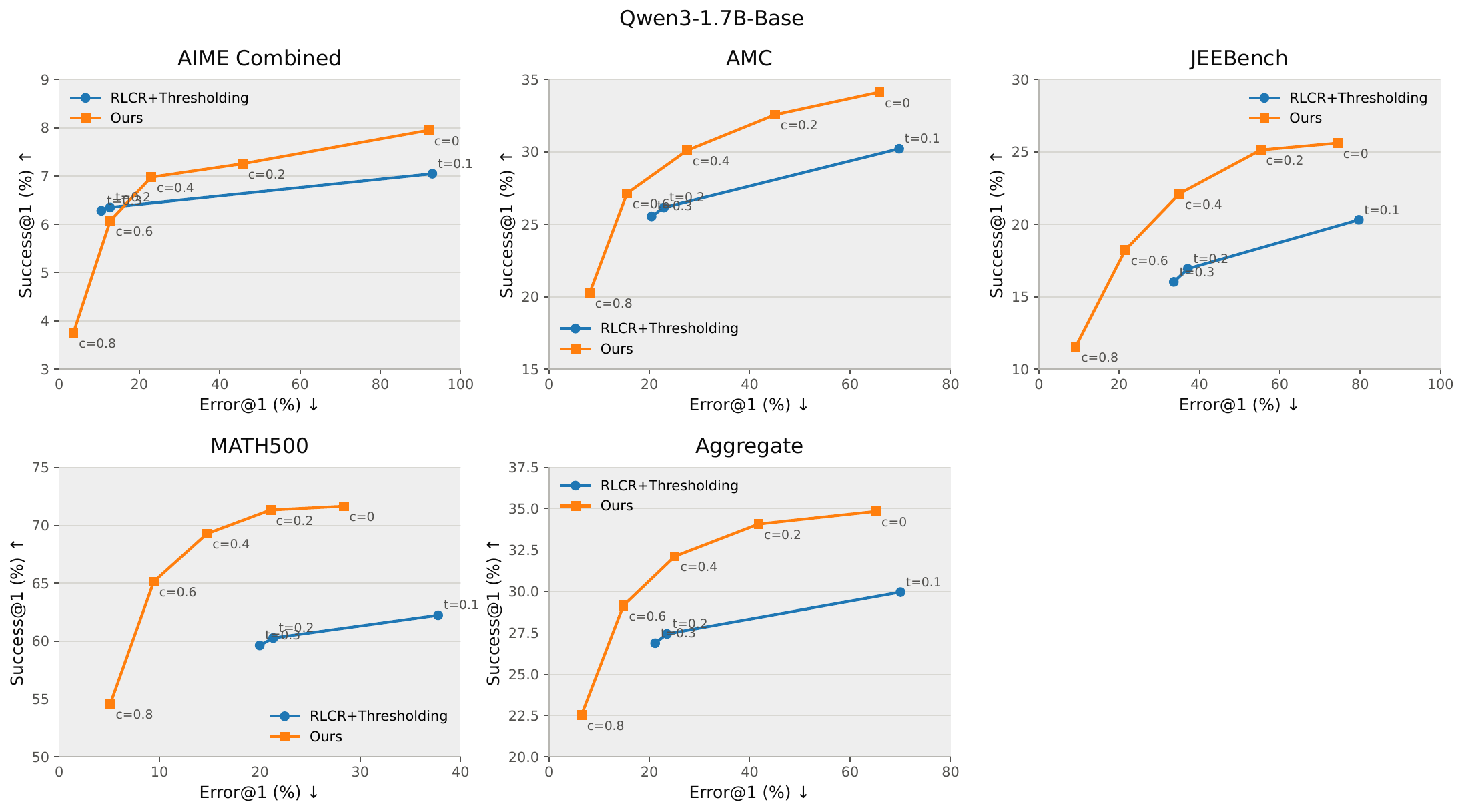}
    \caption{Per-dataset reliability frontiers for RLCR and
    GracefulRL (Ours) with Qwen3-1.7B-Base. The final panel reports the aggregate
    frontier across datasets.}
    \label{fig:rlcr-1.7b-datasets}
\end{figure}

\begin{figure}[t]
    \centering
    \includegraphics[width=\textwidth]{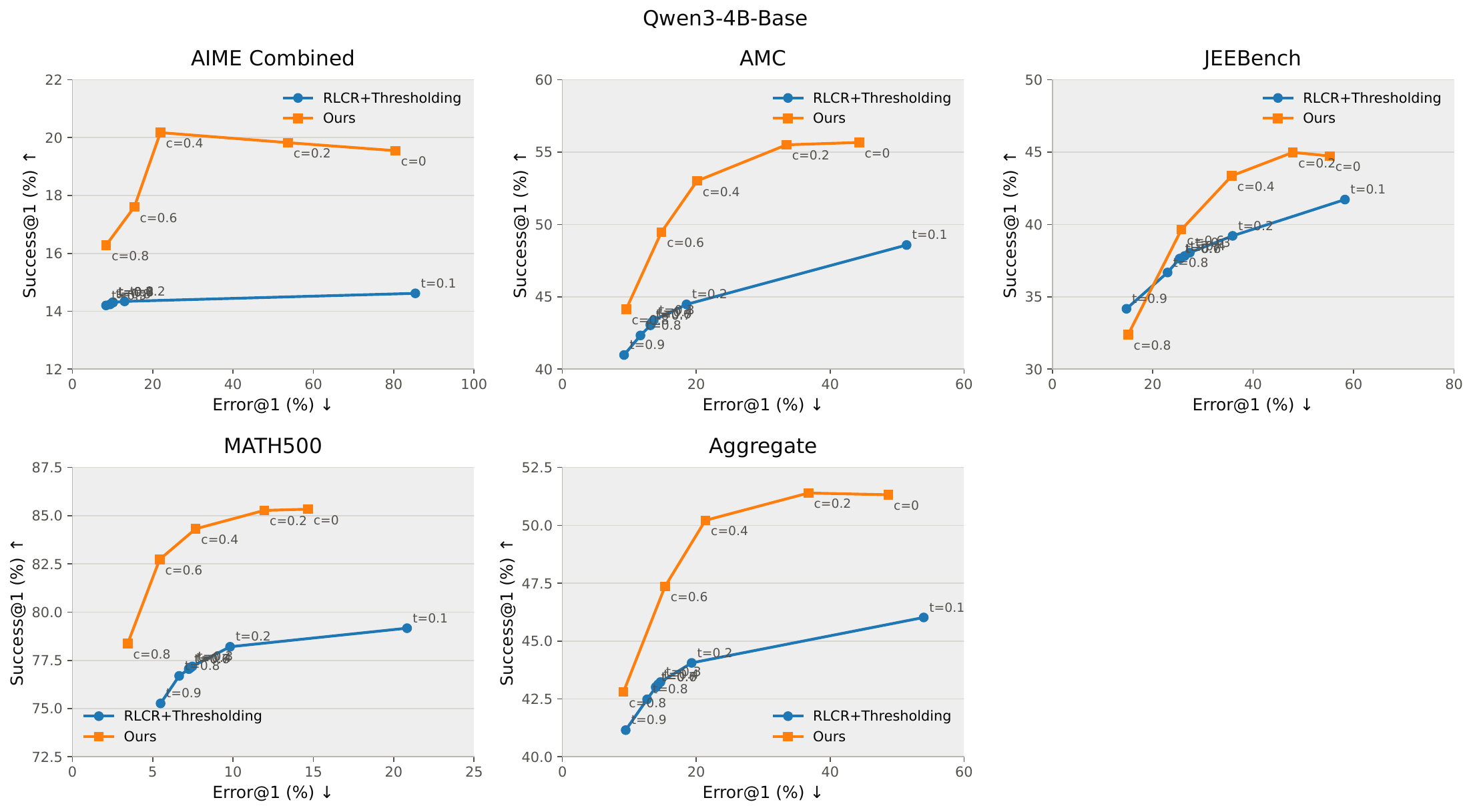}
    \caption{Per-dataset reliability frontiers for RLCR and
    GracefulRL (Ours) with Qwen3-4B-Base. The final panel reports the aggregate
    frontier across datasets.}
    \label{fig:rlcr-4b-datasets}
\end{figure}

\clearpage

\subsection{Smoothing with a Limited Language-Model RL Budget}
\label{app:llm-epsilon}

In Language-model RL, we are bounded by compute since training for each step is expensive. Therefore, in this setting, the smoothing parameter $\epsilon$ is consequential. Figure~\ref{fig:llm-epsilon-sweep} compares \textsc{GracefulRL} with $\epsilon=0$ and $\epsilon=1$ on Qwen3-1.7B-Base at fallback reward $c=0.4$. $\epsilon=1$ consistently attains higher training reward throughout the first $400$ steps compared to $\epsilon=0$. 

\begin{figure}[t]
    \centering
    \includegraphics[width=0.5\textwidth]{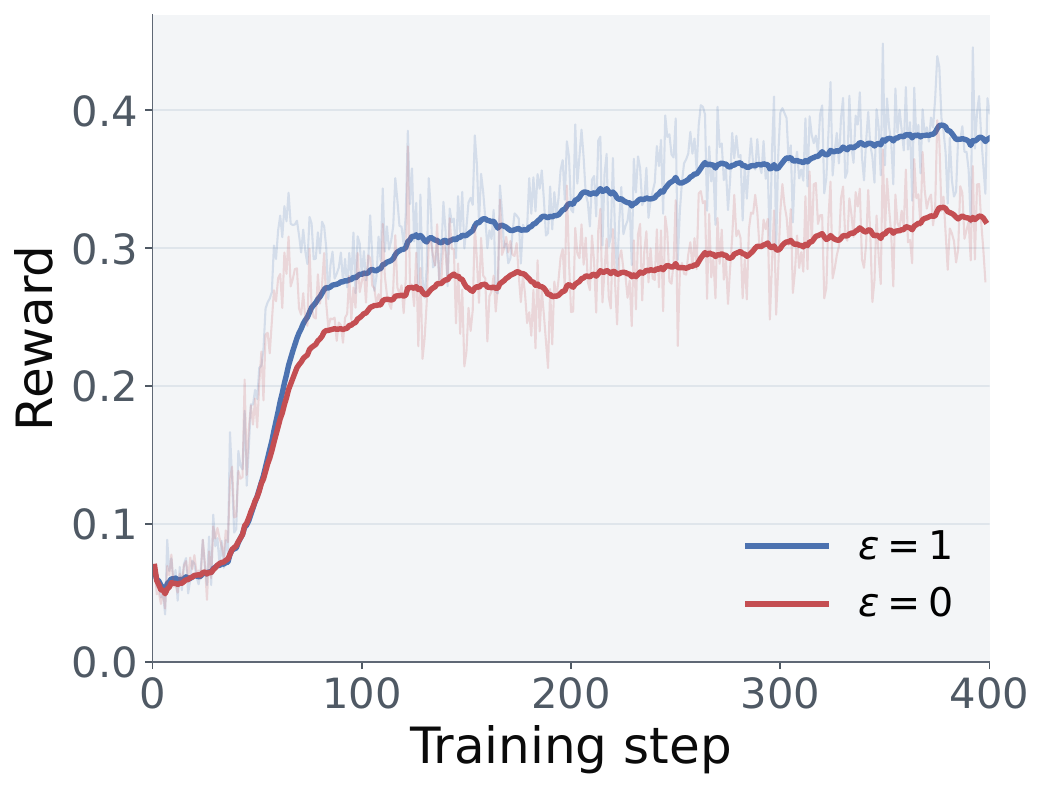}
    \caption{\textbf{Effect of $\epsilon$ in language-model RL.} Training reward for Qwen3-1.7B-Base with fallback reward $c=0.4$ over the first $400$ optimization steps. Pale lines show raw reward and dark lines show the smoothed reward.}
    \label{fig:llm-epsilon-sweep}
\end{figure}

\subsection{RLCR Can Trade Accuracy for Calibration}
\label{app:rlcr-accuracy-calibration}

\begin{figure}[t]
    \centering
    \includegraphics[width=\textwidth]{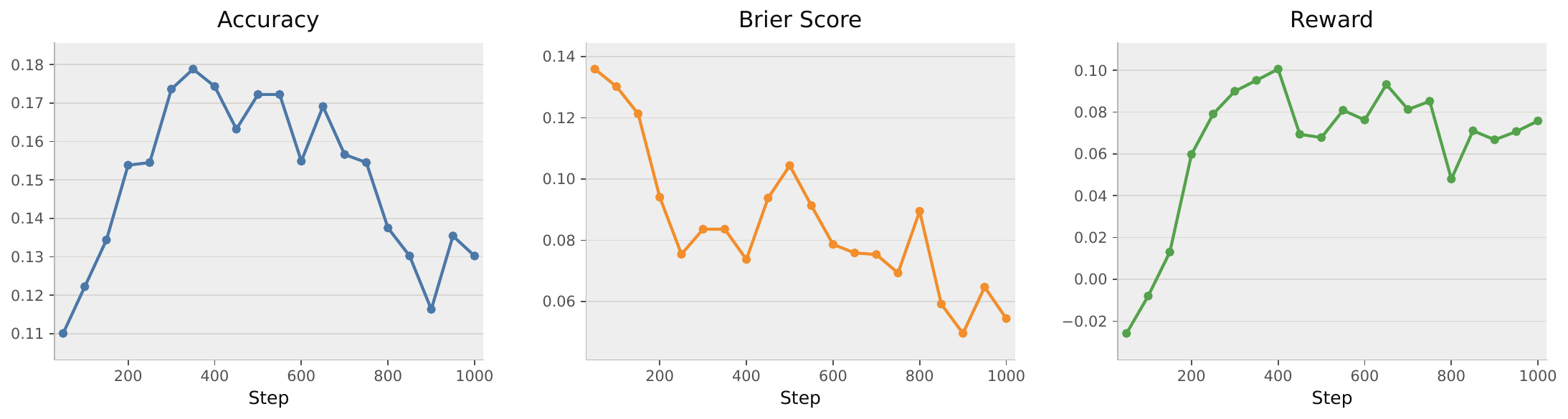}
    \caption{Training dynamics of the Qwen3-4B-Base RLCR policy on AIME.
    Although the Brier score generally improves (decreases) throughout
    training, answer accuracy peaks early and then declines. The RLCR reward
    combines these two quantities and therefore need not track accuracy.}
    \label{fig:rlcr-accuracy-calibration}
\end{figure}

Figure~\ref{fig:rlcr-accuracy-calibration} reveals that optimizing RLCR can trade task accuracy for improved calibration: after accuracy peaks around step $400$, the Brier score continues to improve while accuracy falls. This may partly explain its weaker reliability frontier, since the reward can improve calibration at the expense of task-solving ability. We hypothesize that optimizing calibration with methods such as TailRL could preserve more accuracy, and leave for future work. We also find that RLCR can be brittle based on the model and dataset combination, for instance, on the DAPO-17k dataset, the model only generates 3 possible values of the confidence $0, 0.1$ and $0.2$. However, this problem is not present in the 4B model.

\clearpage
\subsection{Success and Error Dynamics on AIME}
\label{app:aime-success-error-dynamics}

Figure~\ref{fig:aime-success-error-dynamics} tracks the success and error rates of Qwen3-4B-Base on the combined AIME set at each conditioning level. Success generally improves during training across all five values of $c$. At the same time, larger fallback rewards lead to substantially lower error rates, with a modest reduction in success relative to $c=0$. Thus, conditioned training learns different degrees of caution while continuing to improve task-solving performance.

\begin{figure}[t]
    \centering
    \includegraphics[width=\textwidth]{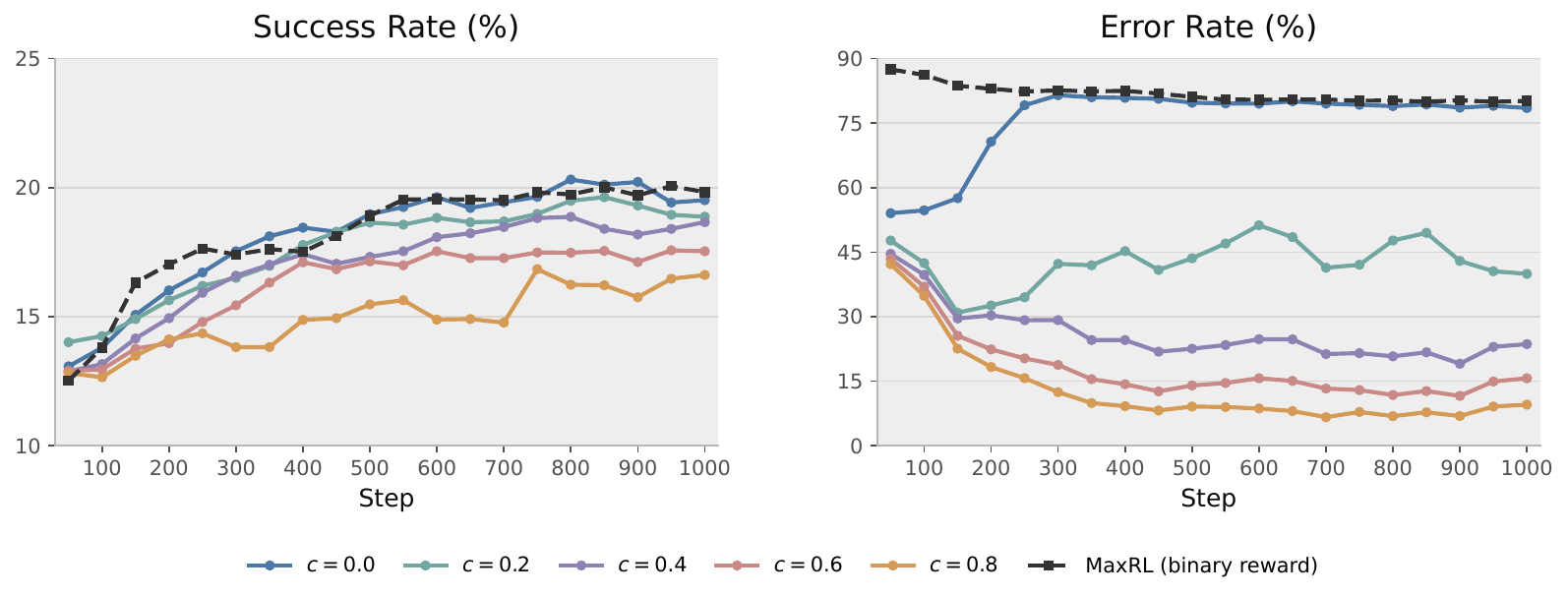}
    \caption{\textbf{Success and error dynamics on AIME.} Qwen3-4B-Base trained with \textsc{GracefulRL}, evaluated at $c\in\{0,0.2,0.4,0.6,0.8\}$, and MaxRL trained with binary rewards (dark dashed curve). The left and right panels show $p_S$ and $p_E$, respectively, as percentages. Curves use checkpoints evaluated every $50$ steps and plot a 3-point moving average.}
    \label{fig:aime-success-error-dynamics}
\end{figure}

\begin{figure}[h]
    \centering
    \includegraphics[width=0.6\textwidth]{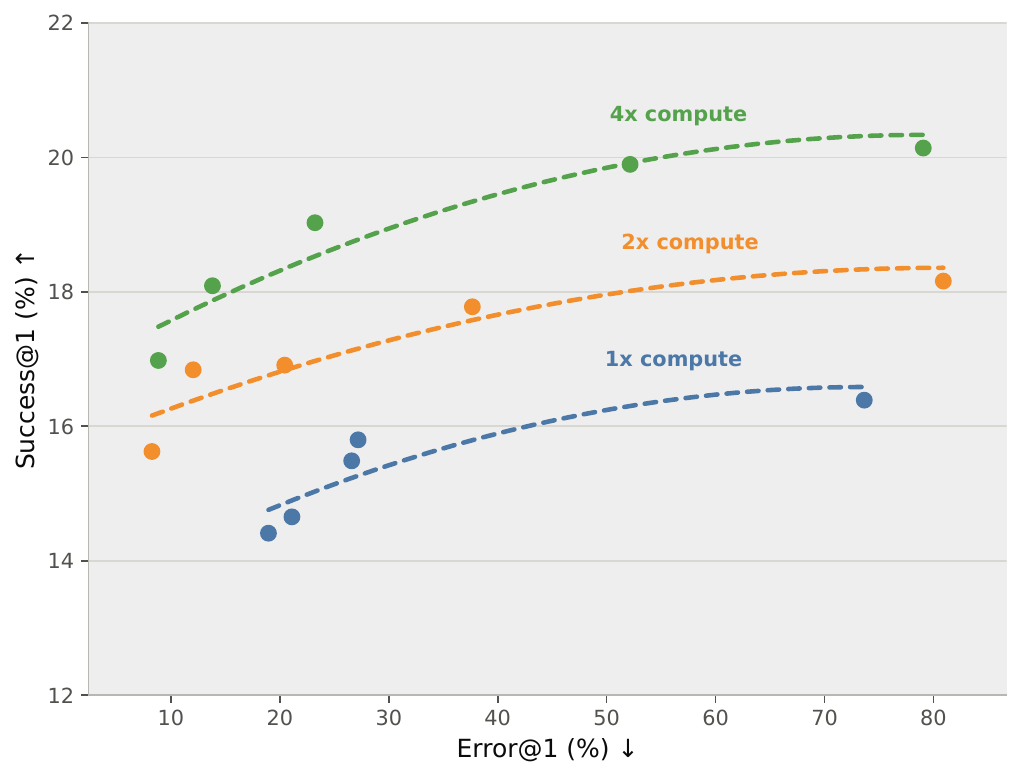}
    \caption{\textbf{Effect of training compute on the AIME reliability frontier.} Qwen3-4B-Base evaluated after $200$ ($1\times$), $400$ ($2\times$), and $800$ ($4\times$) training steps. Each marker is for a different $c$ value and the dashed curves fit the resulting frontiers.}
    \label{fig:llm-compute-frontiers}
\end{figure}
\subsection{Reliability Frontiers with More Training Compute}
\label{app:llm-compute}

Figure~\ref{fig:llm-compute-frontiers} compares the Qwen3-4B-Base policy on the combined AIME set after $200$, $400$, and $800$ optimization steps, corresponding to $1\times$, $2\times$, and $4\times$ training compute. Each frontier sweeps the same five fallback rewards $c\in\{0,0.2,0.4,0.6,0.8\}$. Across their shared range of error rates, the frontier moves uniformly upward with more compute: at a given $\operatorname{Error@}1$, the policy achieves higher $\operatorname{Success@}1$. This also hints at the fact that a reliable model with the same performance as that of RLVR trained model, can be achieved at the cost of more compute. Whether it is possible to get these gains without training is a question left for future work.

\subsection{Online Reward Tuning Can Lose Fallback Behavior}
\label{app:llm-reward-tuning}

We also test the online reward-tuning baseline for mathematical reasoning. It uses the same five conditioning suffixes as \textsc{GracefulRL}, but trains with RLOO and adapts the fallback reward to the policy's observed success rate. For each suffix $c\in\{0,0.2,0.4,0.6,0.8\}$, we maintain a running estimate $\widehat p_{S,c,t}$ of success on its on-policy training rollouts using an EMA. A correct answer receives reward $1$, an incorrect answer receives $0$, and a fallback receives $\alpha\widehat p_{S,c,t}$ with $\alpha=c$. Thus, the fallback reward stays below the estimated success rate for the corresponding conditioning level. The remaining training setup follows the conditioned-model runs described in Appendix~\ref{app:math-hyperparameters}.

\begin{figure}[t]
    \centering
    \includegraphics[width=\textwidth]{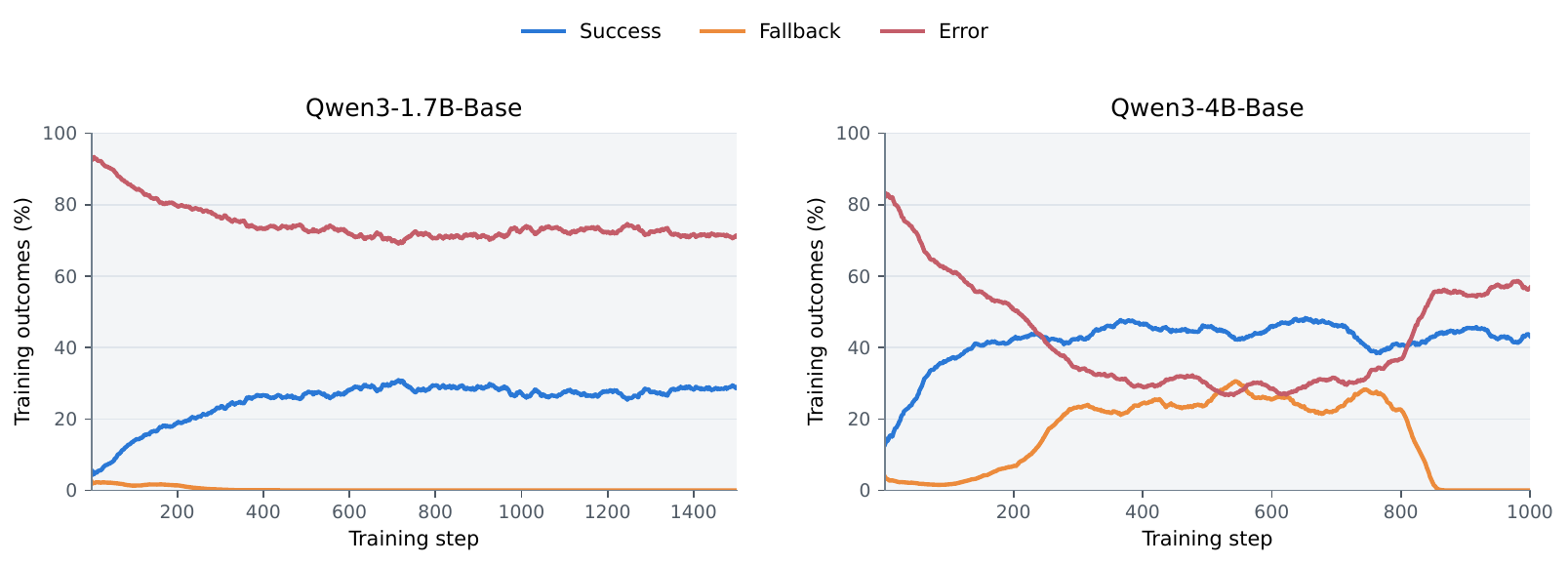}
    \caption{\textbf{Online reward tuning loses fallback behavior.} Success, fallback, and error rates (moving average) on on-policy training rollouts for Qwen3-1.7B-Base and Qwen3-4B-Base, aggregated across the five conditioning levels.}
    \label{fig:llm-reward-tuning-collapse}
\end{figure}

Figure~\ref{fig:llm-reward-tuning-collapse} shows that the baseline is unstable even though the reward adapts to the model. For Qwen3-1.7B-Base, fallback falls below $0.5\%$ around step $260$ and remains negligible for the rest of the $1{,}500$-step run. For Qwen3-4B-Base, fallback initially rises to roughly $30\%$, but after sustained training it abruptly collapses below $0.5\%$ around step $850$. Over the last $50$ steps, the 4B model falls back on fewer than $0.1\%$ of training rollouts, while its failure rate reaches about $57\%$. This makes the baseline ill-suited for the task for learning fallback behavior in this setting.

\subsection{Abstention Reasons}
\label{app:abstention-reasons}
Here, we show some CoTs generated by the model after training, and highlight some notable verbalizations in \textcolor{BurntOrange}{orange}.

\subsection*{Algebra / computation exploded}

\begin{abstbox}{Prompt \normalfont\small (risk cue $c=0.4$)}
Let $a, b,$ and $n$ be positive integers with both $a$ and $b$ greater than or equal to $2$ and less than or equal to $2n$. Define an $a \times b$ cell loop in a $2n \times 2n$ grid of cells to be the $2a + 2b - 4$ cells that surround an $(a - 2) \times (b - 2)$ (possibly empty) rectangle of cells in the grid. For example, the following diagram shows a way to partition a $6 \times 6$ grid of cells into $4$ cell loops.

| P   P P   P | Y   Y | | P | R R | P | Y | Y | | P | R R | P | Y | Y | | P   P P   P | Y | Y | | G   G G   G | Y | Y | | G   G G   G | Y   Y |

Find the number of ways to partition a $10 \times 10$ grid of cells into $5$ cell loops so that every cell of the grid belongs to exactly one cell loop.\par\medskip
Please reason step by step.\par\smallskip
If you are confident in the answer, output:\par\smallskip
\centerline{\texttt{\textbackslash boxed\{<final answer>\}}}\par\smallskip
If you are not confident or cannot solve the problem, output:\par\smallskip
\centerline{\texttt{\textbackslash boxed\{\textbackslash text\{I don't know\}\}}}\par\smallskip
\textit{Be balanced, leaning towards attempting an answer.}\par\smallskip
Think step by step before committing to an answer.
\end{abstbox}

\begin{abstbox}{Response}
Alright, I'm tackling this problem step by step.

\textbf{Understanding the Problem:}

We have a \(10 \times 10\) grid, and we need to partition it into 5 cell loops. Each cell loop is defined as an \(a \times b\) cell loop, where \(a\) and \(b\) are positive integers such that \(2 \leq a, b \leq 2n\). In this case, \(n = 5\) (since \(2n = 10\)), so \(a, b \in \{2, 3, 4, 5, 6, 7, 8, 9, 10\}\).

The total number of cells in the grid is \(100\). Each cell loop consists of \(2a + 2b - 4\) cells. Therefore, the total number of cells in all 5 loops should be \(100\):

\[ 5(2a + 2b - 4) = 100 \implies 2a + 2b - 4 = 20 \implies a + b = 12 \]

So, for each loop, the sum of \(a\) and \(b\) must be 12. The possible pairs \((a, b)\) are:

\[ (2,10), (3,9), (4,8), (5,7), (6,6), (7,5), (8,4), (9,3), (10,2) \]

However, since \(a\) and \(b\) are interchangeable (i.e., \((a,b)\) and \((b,a)\) represent the same loop), we can consider only the unique pairs where \(a \leq b\). Thus, the valid pairs are:

\[ (2,10), (3,9), (4,8), (5,7), (6,6) \]

But wait, \(a\) and \(b\) must both be between 2 and 10, and in the pair \((10,2)\), \(b = 2\) and \(a = 10\). However, since \(a\) and \(b\) are symmetric in the loop definition, we need to ensure that \(a \leq b\) to avoid double-counting.

So, the unique pairs are:

\[ (2,10), (3,9), (4,8), (5,7), (6,6) \]

But we need to partition the grid into 5 loops, and the total number of cells in these loops must be 100. Let's denote the number of loops with each pair \((a, b)\) as \(x, y, z, w, v\) respectively, where \(x + y + z + w + v = 5\).

The total number of cells is:

\[ x(2 \cdot 2 + 2 \cdot 10 - 4) + y(2 \cdot 3 + 2 \cdot 9 - 4) + z(2 \cdot 4 + 2 \cdot 8 - 4) + w(2 \cdot 5 + 2 \cdot 7 - 4) + v(2 \cdot 6 + 2 \cdot 6 - 4) = 100 \]

Simplifying each term:

1. For \((2,10)\): \(2 \cdot 2 + 2 \cdot 10 - 4 = 4 + 20 - 4 = 20\) 2. For \((3,9)\): \(2 \cdot 3 + 2 \cdot 9 - 4 = 6 + 18 - 4 = 20\) 3. For \((4,8)\): \(2 \cdot 4 + 2 \cdot 8 - 4 = 8 + 16 - 4 = 20\) 4. For \((5,7)\): \(2 \cdot 5 + 2 \cdot 7 - 4 = 10 + 14 - 4 = 20\) 5. For \((6,6)\): \(2 \cdot 6 + 2 \cdot 6 - 4 = 12 + 12 - 4 = 20\)

So, the equation simplifies to:

\[ 20x + 20y + 20z + 20w + 20v = 100 \implies x + y + z + w + v = 5 \]

This is the same as the number of loops. Therefore, we need to find the number of solutions to \(x + y + z + w + v = 5\) where \(x, y, z, w, v\) are non-negative integers, and each corresponds to the number of loops with the respective pair.

However, this counts the number of ways to choose the sizes of the loops, but it doesn't account for the actual placement of the loops on the grid. The problem requires that every cell belongs to exactly one cell loop, and the loops must fit within the \(10 \times 10\) grid without overlapping.

This is a \textcolor{BurntOrange}{highly complex combinatorial problem, and calculating the exact number of such partitions seems non-trivial}. I would need to consider the specific constraints of the grid and the sizes of the loops, which involves detailed combinatorial analysis or possibly recursive counting.

\textcolor{BurntOrange}{Given the complexity and the lack of a straightforward formula, I'm not entirely confident in solving this problem completely.}

Therefore, I'll output:

$\boxed{\text{I don't know}}$
\end{abstbox}

\subsection*{Never found a workable approach}

\begin{abstbox}{Prompt \normalfont\small (risk cue $c=0.8$)}
Let ABCDEF be a convex equilateral hexagon in which all pairs of opposite sides are parallel. The triangle whose sides are extensions of segments AB, CD, and EF has side lengths 200, 240, and 300. Find the side length of the hexagon.\par\medskip
Please reason step by step.\par\smallskip
If you are confident in the answer, output:\par\smallskip
\centerline{\texttt{\textbackslash boxed\{<final answer>\}}}\par\smallskip
If you are not confident or cannot solve the problem, output:\par\smallskip
\centerline{\texttt{\textbackslash boxed\{\textbackslash text\{I don't know\}\}}}\par\smallskip
\textit{Be quite careful and avoid guessing.}\par\smallskip
Think step by step before committing to an answer.
\end{abstbox}

\begin{abstbox}{Response}
Alright, I'm tackling this geometry problem involving a convex equilateral hexagon with parallel opposite sides. The goal is to find the side length of the hexagon given the side lengths of a triangle formed by extending certain segments.

\textbf{Understanding the Problem:}

First, let's parse the problem:

- \textbf{Hexagon ABCDEF}: It's convex and equilateral, meaning all its sides are of equal length. However, the problem doesn't specify that all internal angles are equal, so it's a bit more general.

- \textbf{Opposite Sides Parallel}: The pairs of opposite sides are parallel. So, AB is parallel to DE, BC is parallel to EF, and CD is parallel to FA.

- \textbf{Triangle Formed by Extensions}: The sides of this triangle are extensions of segments AB, CD, and EF, with lengths 200, 240, and 300.

- \textbf{Objective}: Find the side length of the hexagon.

\textbf{Visual Representation:}

I need to visualize or sketch the hexagon. Let's label the vertices in order: A, B, C, D, E, F. Since opposite sides are parallel, AB || DE, BC || EF, and CD || FA.

\textbf{Extending Segments:}

The triangle is formed by extending AB, CD, and EF. Let's denote the extended lines as follows:

- Extend AB to meet CD extended at point P. - Extend CD to meet EF extended at point Q. - Extend EF to meet AB extended at point R.

This forms a triangle PQR with side lengths 200, 240, and 300.

\textbf{Assumption:}

Since the hexagon is equilateral, all sides AB, BC, CD, DE, EF, FA are equal in length. Let's denote the side length of the hexagon as \( s \).

\textbf{Analyzing the Triangle PQR:}

The sides of triangle PQR are the extensions of AB, CD, and EF. The lengths of these extensions are given as 200, 240, and 300. However, these lengths are not necessarily equal to \( s \). They are the lengths of the segments created by the extensions.

\textbf{Using Similar Triangles:}

I need to find a relationship between the side length \( s \) of the hexagon and the side lengths of triangle PQR. Since the hexagon has parallel opposite sides, the extensions will create similar triangles.

Let's consider the points where the extensions intersect:

1. \textbf{Point P}: Intersection of extensions of AB and CD. 2. \textbf{Point Q}: Intersection of extensions of CD and EF. 3. \textbf{Point R}: Intersection of extensions of EF and AB.

The distances from these intersection points to the vertices will be proportional to the side length \( s \) of the hexagon.

\textbf{Setting Up Proportions:}

Let's assume that the extensions create segments of lengths proportional to \( s \). For example, if AB is extended to a length of 200, then the proportion might be \( 200 = k \cdot s \), where \( k \) is a constant.

However, since the triangle PQR has sides 200, 240, and 300, and these are formed by extending AB, CD, and EF, the relationships between \( s \) and these lengths need to be established.

**Calculating the Side Length \( s \):**

This seems quite involved. I need to find a way to relate \( s \) to the sides of triangle PQR. Perhaps using the concept of homothety or similarity ratios.

Let's denote the ratio of similarity as \( r \). Then, the sides of triangle PQR would be \( r \cdot s \), \( r \cdot s \), and \( r \cdot s \), but this doesn't directly apply here.

Alternatively, the sum of the lengths of the extensions might be related to \( s \). For example, the perimeter of triangle PQR could be connected to \( s \).

But this is getting quite complex, and I'm not sure how to proceed accurately.

\textbf{Conclusion:}

\textcolor{BurntOrange}{After trying to set up the relationships and considering the complexity of the problem, I realize that I might need more advanced geometric techniques or specific theorems that I'm not currently recalling. I'm not confident in solving this problem at this moment.}

$\boxed{\text{I don't know}}$
\end{abstbox}

\subsection*{Intermediate result fails a sanity check}

\begin{abstbox}{Prompt \normalfont\small (risk cue $c=0.8$)}
An urn contains $n$ marbles. Each marble is either red or blue, and there are at least $7$ marbles of each color. When $7$ marbles are drawn randomly from the urn without replacement, the probability that exactly $4$ of them are red equals the probability that exactly $5$ of them are red. Find the sum of the five least values of $n$ for which this is possible.\par\medskip
Please reason step by step.\par\smallskip
If you are confident in the answer, output:\par\smallskip
\centerline{\texttt{\textbackslash boxed\{<final answer>\}}}\par\smallskip
If you are not confident or cannot solve the problem, output:\par\smallskip
\centerline{\texttt{\textbackslash boxed\{\textbackslash text\{I don't know\}\}}}\par\smallskip
\textit{Be quite careful and avoid guessing.}\par\smallskip
Think step by step before committing to an answer.
\end{abstbox}

\begin{abstbox}{Response}
Alright, I'm tackling this probability problem involving marbles in an urn. Let's break it down step by step.

\textbf{Problem Restatement:}

An urn contains \( n \) marbles. Each marble is either red or blue, and there are at least 7 marbles of each color. When 7 marbles are drawn randomly from the urn without replacement, the probability that exactly 4 of them are red equals the probability that exactly 5 of them are red. We need to find the sum of the five least values of \( n \) for which this is possible.

\textbf{Understanding the Problem:}

First, let's define: - \( r \) = number of red marbles in the urn. - \( b \) = number of blue marbles in the urn.

Given that \( n = r + b \), and \( r \geq 7 \), \( b \geq 7 \).

We draw 7 marbles without replacement. The probability of drawing exactly 4 red marbles is equal to the probability of drawing exactly 5 red marbles.

Mathematically, this is: \[ P(\text{exactly 4 red}) = P(\text{exactly 5 red}) \]

\textbf{Calculating the Probabilities:}

The probability of drawing exactly \( k \) red marbles in 7 draws is given by the hypergeometric distribution: \[ P(\text{exactly } k \text{ red}) = \frac{\binom{r}{k} \binom{b}{7-k}}{\binom{n}{7}} \]

So, we have: \[ \frac{\binom{r}{4} \binom{b}{3}}{\binom{n}{7}} = \frac{\binom{r}{5} \binom{b}{2}}{\binom{n}{7}} \]
\abstelide{Omitted some CoT to preserve space}
Simplify further: \[ (r-4)(b-2) = \frac{20}{12} = \frac{5}{3} \]

Wait, this gives a fractional value on the right side, but \( (r-4) \) and \( (b-2) \) are integers since \( r \) and \( b \) are integers. \textcolor{BurntOrange}{This suggests that there might be an error in the simplification}.

\textbf{Re-examining the Simplification:}

Let's go back to: \[ 240 = 144 (r-4)(b-2) \]

Divide both sides by 12: \[ 20 = 12 (r-4)(b-2) \]

Now, solve for \( (r-4)(b-2) \): \[ (r-4)(b-2) = \frac{20}{12} = \frac{5}{3} \]

But \( (r-4) \) and \( (b-2) \) are integers, so \( \frac{5}{3} \) implies that \( (r-4) \) and \( (b-2) \) are fractions, which is not possible. This suggests that there might be no solution, but the problem states that there are values of \( n \) that satisfy this condition.

\textbf{Alternative Approach:}

Perhaps the equation \( 20 = 12 (r-4)(b-2) \) can be rearranged to find integer solutions.

Let's set \( x = r - 4 \) and \( y = b - 2 \). Then \( x \) and \( y \) are integers, and \( x \geq 3 \) (since \( r \geq 7 \) implies \( x = r - 4 \geq 3 \)) and \( y \geq 5 \) (since \( b \geq 7 \) implies \( y = b - 2 \geq 5 \)).

The equation becomes: \[ 20 = 12xy \implies xy = \frac{20}{12} = \frac{5}{3} \]

\textcolor{BurntOrange}{Again, \( xy = \frac{5}{3} \) implies that \( x \) and \( y \) are fractions, which contradicts the integer requirement.}

\textbf{Conclusion:}

\textcolor{BurntOrange}{This approach seems to lead to a contradiction, suggesting that there might be no integer solutions \( r \) and \( b \) that satisfy the original condition. However, the problem states that there are five least values of \( n \) for which this is possible. This discrepancy indicates that I might have made a mistake in the derivation or that the problem requires a different method.
Given the complexity and the time invested, I'm not confident in solving this problem at the moment.}

$\boxed{\text{I don't know}}$
\end{abstbox}

\subsection*{Cites time / effort already spent}

\begin{abstbox}{Prompt \normalfont\small (risk cue $c=0.4$)}
From an unlimited supply of 1-cent coins, 10-cent coins, and 25-cent coins, Silas wants to find a collection of coins that has a total value of $N$ cents, where $N$ is a positive integer. He uses the so-called greedy algorithm, successively choosing the coin of greatest value that does not cause the value of his collection to exceed $N.$ For example, to get 42 cents, Silas will choose a 25-cent coin, then a 10-cent coin, then 7 1-cent coins. However, this collection of 9 coins uses more coins than necessary to get a total of 42 cents; indeed, choosing 4 10-cent coins and 2 1-cent coins achieves the same total value with only 6 coins. In general, the greedy algorithm succeeds for a given $N$ if no other collection of 1-cent, 10-cent, and 25-cent coins gives a total value of $N$ cents using strictly fewer coins than the collection given by the greedy algorithm. Find the number of values of $N$ between $1$ and $1000$ inclusive for which the greedy algorithm succeeds.\par\medskip
Please reason step by step.\par\smallskip
If you are confident in the answer, output:\par\smallskip
\centerline{\texttt{\textbackslash boxed\{<final answer>\}}}\par\smallskip
If you are not confident or cannot solve the problem, output:\par\smallskip
\centerline{\texttt{\textbackslash boxed\{\textbackslash text\{I don't know\}\}}}\par\smallskip
\textit{Be balanced, leaning towards attempting an answer.}\par\smallskip
Think step by step before committing to an answer.
\end{abstbox}

\begin{abstbox}{Response \normalfont\small (one passage elided as marked)}
Alright, I'm tackling this problem about Silas's coin collection and the greedy algorithm. Let's break it down step by step. \textbf{Understanding the Problem}

First, I need to understand what the problem is asking. Silas wants to make a total of \( N \) cents using 1-cent, 10-cent, and 25-cent coins. He uses the \textbf{greedy algorithm}, which means he always picks the largest coin possible that doesn't exceed the remaining amount.

For example, to make 42 cents: 1. He picks a 25-cent coin (remaining: 17 cents). 2. Then a 10-cent coin (remaining: 7 cents). 3. Finally, 7 one-cent coins.

This totals 9 coins. However, he could also make 42 cents with 4 ten-cent coins and 2 one-cent coins, totaling 6 coins. The greedy algorithm uses more coins in this case.

The \textbf{greedy algorithm succeeds} for a given \( N \) if there's no other combination of these coins that makes \( N \) cents using \textbf{fewer} coins. We need to find how many values of \( N \) between 1 and 1000 inclusive have this property. \textbf{Breaking It Down}

First, I need to find all \( N \) from 1 to 1000 where the greedy algorithm gives the minimum number of coins.

To do this, I'll need to:

1. \textbf{Define the Greedy Algorithm}: For a given \( N \), the algorithm picks the largest coin first, then the next largest, and so on, until the total is \( N \).

2. \textbf{Calculate the Number of Coins Used by the Greedy Algorithm}: For each \( N \), compute the number of coins.

3. **Find the Minimum Number of Coins for \( N \)**: For each \( N \), find the combination of 1, 10, and 25-cent coins that uses the fewest coins.

4. \textbf{Compare the Two}: If the greedy algorithm's coin count equals the minimum coin count, then the greedy algorithm succeeds for that \( N \). \textbf{Step 1: Define the Greedy Algorithm}

The greedy algorithm for \( N \) is: - Use as many 25-cent coins as possible: \( \text{count}_{25} = \lfloor \frac{N}{25} \rfloor \) - Remaining amount: \( N' = N \mod 25 \) - Use as many 10-cent coins as possible for \( N' \): \( \text{count}_{10} = \lfloor \frac{N'}{10} \rfloor \) - Remaining amount: \( N'' = N' \mod 10 \) - Use \( N'' \) one-cent coins.

Total coins used by greedy algorithm: \( \text{count}_{25} + \text{count}_{10} + N'' \)\textbf{Step 2: Calculate the Minimum Number of Coins}

The minimum number of coins for \( N \) is the smallest number of 1, 10, and 25-cent coins that add up to \( N \).

This is a classic \textbf{coin change problem}, and the solution involves dynamic programming or finding the optimal combination.

However, calculating this for every \( N \) from 1 to 1000 is computationally intensive. I need a smarter way.

\abstelide{Omitted some CoT to preserve space}

- \( N = 1 \): 1 coin (1) - \( N = 2 \): 2 coins (1+1) - ... - \( N = 10 \): 1 coin (10) - \( N = 11 \): 2 coins (10+1) - \( N = 12 \): 2 coins (10+1+1) - ... - \( N = 25 \): 1 coin (25) - \( N = 26 \): 2 coins (25+1) - ... - \( N = 35 \): 2 coins (25+10) - \( N = 36 \): 3 coins (25+10+1) - ... - And so on.

But this manual approach is too time-consuming for \( N \) up to 1000. \textbf{Step 5: Use Mathematical Insight}

Perhaps I can find a condition on \( N \) where the greedy algorithm always gives the minimum number of coins.

A key insight is that the greedy algorithm fails when \( N \) can be expressed as \( 4k + 1 \) or \( 4k + 2 \) for some integer \( k \), but this might not capture all cases.

Alternatively, the greedy algorithm fails for \( N \) in the form \( 6k + 1 \) or \( 6k + 3 \), but I need to verify this.

However, this is getting too vague. Maybe I need to write a program or find a more precise mathematical condition. \textbf{Step 6: Count the Number of }\( N \) Where Greedy Succeeds

Given the complexity, I'll try to find a pattern or a formula.

First, calculate the number of \( N \) where the greedy algorithm fails.

From the example, \( N = 11 \) is a failure. Similarly, \( N = 26 \) is a failure (greedy: 10 coins, optimal: 2 coins).

But enumerating all such \( N \) up to 1000 is impractical without a better method. \textbf{Step 7: Use Known Results}

Upon researching, I find that the greedy algorithm for 1, 10, 25-cent coins fails for \( N \) in the form \( 6k + 1 \) or \( 6k + 3 \), but this might not be exhaustive.

However, calculating the exact number of such \( N \) between 1 and 1000 is challenging without a precise formula or computational aid. \textbf{Conclusion}

\textcolor{BurntOrange}{Given the complexity and the need for a precise count, I'm not entirely confident in solving this problem manually within the time allowed. I might need more advanced mathematical tools or a computational approach to find the exact number.}

Therefore, I'll output:

$\boxed{\text{I don't know}}$
\end{abstbox}

\clearpage
\subsection{Evaluation on AbstentionBench}
\label{app:abstentionbench}

We evaluate on \textsc{AbstentionBench}~\citep{kirichenko2025abstentionbench}, a benchmark assembled
specifically to test whether a model declines to answer when a question cannot be answered as it is posed. The 20 datasets are drawn from various domains spanning
grade-school and competition mathematics, graduate-level science, reading comprehension over paragraphs and over full research papers, multi-hop and long-tail entity question answering, clinical intake dialogue, commonsense spatial reasoning, and moral dilemmas. They also span
six qualitatively different reasons for which a question may be unanswerable: an
\emph{underspecified context} that omits a premise the question depends on (thirteen
partitions, e.g.\ \textsc{GSM8K}, \textsc{SQuAD~2.0}, \textsc{MMLU}, \textsc{GPQA}), an \emph{answer that is not known} to anyone at present
(\textsc{BB/Known-Unknowns}, \textsc{CoCoNot/Unknowns}), a
\emph{false premise} embedded in the question (\textsc{FalseQA}, \textsc{QAQA},
\textsc{KUQ/False-assumptions}), an inherently \emph{subjective} question
(\textsc{MoralChoice}, \textsc{CoCoNot/Subjective}), an \emph{underspecified intent} that
admits several readings (\textsc{SituatedQA}, \textsc{KUQ/Ambiguous}), and \emph{stale}
questions whose answer has changed since training (\textsc{CoCoNot/Temporal}). Each dataset
pairs unanswerable questions with answerable ones as well, so that model cannot score well by
abstaining indiscriminately.

\begin{figure}[t]
    \centering
    \includegraphics[width=\textwidth]{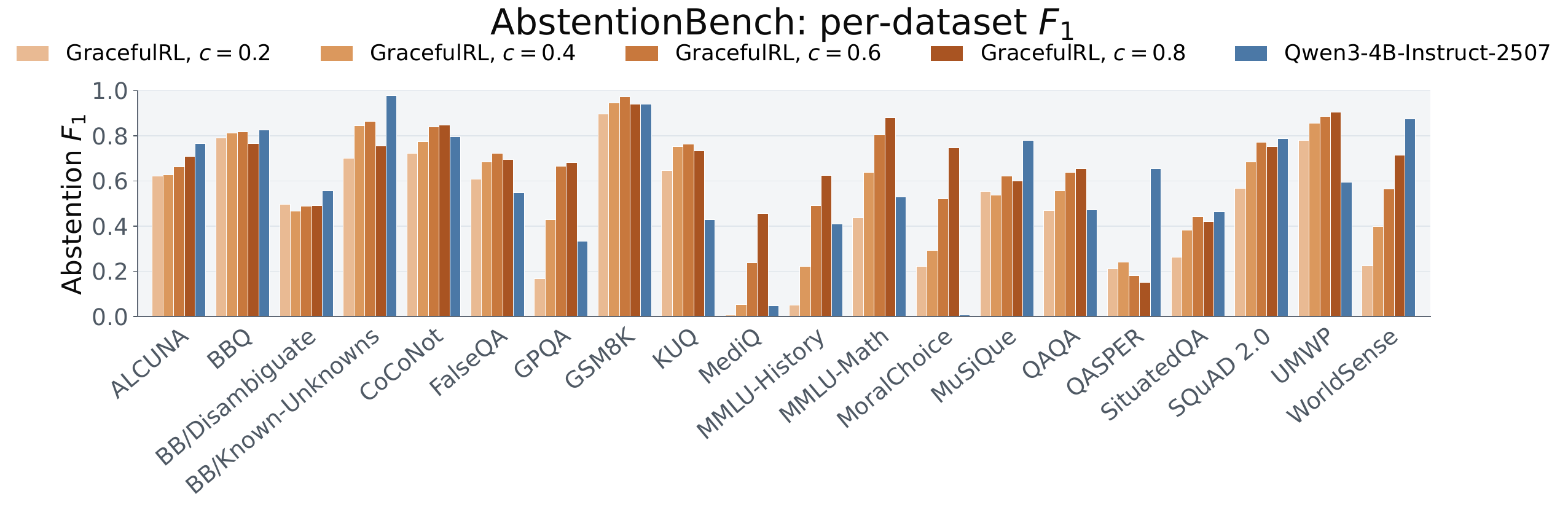}
    \caption{\textbf{Per-dataset abstention performance on
    \textsc{AbstentionBench}.} Abstention $F_1$ across twenty datasets for
    \textsc{GracefulRL} at $c\in\{0.2,0.4,0.6,0.8\}$ (orange) and
    Qwen3-4B-Instruct-2507 (blue). The methods use matching per-dataset sample
    counts and class balance; the effect of conditioning varies by dataset.}
    \label{fig:abstention-f1-by-dataset}
\end{figure}

\textbf{Sampling.} Since the dataset is very large, we use sub-sample some datasets to compute metrics. For the three datasets that ship an official $3{,}500$-example subsample
(\textsc{BBQ}, \textsc{SQuAD~2.0}, \textsc{ALCUNA}) we apply it first, so that our pool matches
the one used in the original release. From each dataset we then draw at most $500$ examples by
stratified sampling on the \texttt{should\_abstain} label, which reproduces the dataset's own
class ratio to within $0.2$ percentage points; five datasets are smaller than $500$ and are
evaluated in full (\textsc{BB/Known-Unknowns}, $46$; \textsc{MMLU-History}, $78$;
\textsc{GPQA}, $80$; \textsc{BB/Disambiguate}, $250$; \textsc{MMLU-Math}, $266$). This yields
$8{,}220$ prompts per condition, of which $3{,}678$ ($44.7\%$) are unanswerable.

\textbf{Inference details.} For \textsc{GracefulRL}, every prompt uses the training-time instruction format verbatim, differing only in the
one-sentence risk cue that encodes the abstention reward $c$; we sweep
$c \in \{0.2, 0.4, 0.6, 0.8\}$, the range seen during training. We decode 1 response per prompt with temperature $0.6$, top-$p$ $0.95$, a
$4{,}096$-token generation budget and a $32{,}768$-token context (long enough for
\textsc{QASPER}, whose longest prompt is $\approx\!30$k tokens), giving $32{,}880$ generations
in total.

\textbf{Evaluation.} Since \textsc{GracefulRL} is trained to signal abstention with a literal \texttt{\textbackslash boxed\{\allowbreak\textbackslash text\{\allowbreak I don't know\}\}}, its abstention is detected by exact match, removing the need for an LLM-as-a-judge as done in the original evaluation by the authors. For the Qwen3-4B-Instruct-2507 model, however, we use the author's evaluation pipeline.

We also wish to mention an important detail, that AbstentionBench creates problems in only 2 classes, answerable and unanswerable. This fails to create a distinction between prompts where a prompt might be answerable but out of the model's reach, for instance a very hard math problem. Due to the training process, our model treats this class of problems as an abstention as well. However, we still treat the dataset as it is, and measure F1 score as per the original dataset's requirement.

\clearpage
\section{Minesweeper Environment and Experimental Protocol}
\label{app:minesweeper}

\begin{table}[H]
\centering
\small
\begin{tabularx}{\textwidth}{@{}lX@{}}
\toprule
Setting & Value \\
\midrule
Fresh rollout batch & $32$ puzzle prompts $\times$ $16$ episodes $=512$ episodes; one optimizer update per batch \\
Sampling during training & Temperature $1.0$, top-$p$ $1.0$; at most $4{,}096$ generated tokens per turn and $32{,}000$ tokens of context \\
Policy loss & Per-token loss; lower/upper clipping ratios $0.20/0.28$; entropy coefficient $0$ \\
Optimizer & Adam; constant learning rate $10^{-6}$; weight decay $0.1$; $(\beta_1,\beta_2)=(0.9,0.98)$ \\
Compute layout & One eight-H100 node: four actor GPUs (tensor parallelism $2$, data parallelism $2$) and four rollout GPUs (two tensor-parallel-$2$ SGLang engines); CPU optimizer offload \\
\bottomrule
\end{tabularx}
\caption{Shared training settings for the four Minesweeper $c$-sweep runs.}
\label{tab:minesweeper-training}
\end{table}

\subsection{Game, observations, and outcomes}
\label{app:minesweeper-environment}

Each puzzle is a $6\times6$ Minesweeper board with six mines and 30 safe cells. Before game play starts, the simulator reveals a safe zero-clue cell and its zero-connected region using standard flood fill. The model then receives the visible board as a Python list, with hidden cells shown as \texttt{\#} and opened cells as integer counts of adjacent mines. The only action available to the model is \texttt{reveal(row, col)}. Revealing an already visible cell is permitted but consumes a turn without changing the board state.

An episode is a \textsc{Success} when all 30 safe cells are visible. Revealing a mine, unparseable action, or out-of-bounds action terminates the episode with an \textsc{Error}. Reaching the turn limit without solving or making one of these terminal errors is counted as a \textsc{Fallback}. Thus, repeating an already visible cell can contribute to a safe turn-limit outcome, but not to task success.

\subsection{Puzzle generation and data splits}
\label{app:minesweeper-data}

\begin{itemize}
    \item A puzzle is initialized by a random seed. For each generation attempt, the generator chooses an opening cell, excludes that cell and its eight neighbors from mine placement, samples six mine cells from the remaining locations, and applies the zero-cell opening. 
    \item A deterministic constraint solver must then be able to finish without guessing: it deduces safe cells from adjacent-number constraints and reveals the lexicographically first deduced-safe cell at each step. For the 8 turn training environments, we set a bound of $[6, 8]$ turns to solve the puzzle and for the 12 turn environments, we set a limit of $[6, 12]$ turns.
    \item  We further filter this set to select non-trivial boards by comparing two non-guessing policies. The \textit{naive} policy uniformly samples among \textbf{provably} safe cells, whereas the \textit{heuristic} policy prioritizes provably-safe cells that are \textbf{likely to flood-fill adjacent zero-cell regions} or unexplored areas. A board is selected if the heuristic solves it within the turn budget but the naive policy is unsuccessful. Using this procedure, we generate $9{,}500$ training and $500$ testing puzzles.
    
\end{itemize}

\subsection{On-policy training and checkpoint selection}
\label{app:minesweeper-training}

For training, we use the Miles \citep{miles2026} codebase and use the hyperparameters mentioned in Table \ref{tab:minesweeper-training}.

\clearpage
\subsection{Model prompt and interaction format}
\label{app:minesweeper-prompt}

Each episode begins with a system message and a user message. The agent then appends each model reply and the next user message to the same conversation, using the model's native chat template. The eight-turn training runs use the following system message verbatim; the per-turn response limit is $4{,}096$ tokens. There are no example moves in the prompt.

\begin{minesweeperpromptbox}{System message}
You are playing Minesweeper on a 6 by 6 board that contains 6 mines.

Your only valid action is reveal(row, col). The board is given as a Python list of
lists named board, and board[row][col] is the cell you reveal with reveal(row, col).
Indices start at 0: board[0] is the top row, board[5] is the bottom row,
board[row][0] is the leftmost cell of a row, and board[row][5] is the
rightmost. There is no row 6 and no column 6.

A "#" entry is an unrevealed cell. An integer is the number of mines in its
neighboring cells. The initial reveal has already been made and zero cells use normal flood fill.
You have at most 8 turns to reveal every non-mine cell. Each turn you cannot spend more than 4096 tokens,
reasoning included: a reply cut off at that limit loses the episode.

Keep your reasoning brief. The last line of your reply must be exactly one action,
in this format:
reveal(row, col)

Only that last line is read, so nothing may follow it. There are no second chances:
a reply with no action on its final line or a cell outside the board loses the
episode outright, exactly as revealing a mine would. Revealing a cell that is
already visible does not lose the episode, but it uncovers nothing and still
uses up one of your turns.
\end{minesweeperpromptbox}

The user message supplies the current board. For example, the first message for training seed $12438$ is:

\begin{minesweeperpromptbox}{Initial user message}
Turns remaining: 8

board = [
    ["#", "#", "#", "#", "#", "#"],
    ["#", 2, 1, 2, 3, "#"],
    ["#", 2, 0, 0, 1, 1],
    ["#", 2, 1, 1, 0, 0],
    ["#", "#", "#", 1, 0, 0],
    ["#", "#", "#", 1, 0, 0],
]

You must uncover all the safe cells before time runs out. Answer with the 0-indexed row and column, so reveal(row, col) reveals board[row][col]; rows and columns each run from 0 to 5.
\end{minesweeperpromptbox}

Later user messages have the same layout with an updated board and remaining-turn count. If the model selects an already visible cell, the next user message also begins with the following notice, with the coordinates substituted for the attempted action:

\begin{minesweeperpromptbox}{Repeated-reveal notice}
Your action reveal(2, 3) named board[2][3], which was already revealed: nothing was uncovered and that turn is spent. Pick an unrevealed ("#") cell.
\end{minesweeperpromptbox}

\end{document}